%% file: main.tex
\documentclass[final,12pt]{colt2022} 
\usepackage[utf8]{inputenc}

\input{style.sty}
\usepackage{pgfplots}
\pgfplotsset{compat=1.17}
\usepackage{tikz}
\usetikzlibrary{decorations.pathreplacing}

\title[On the benefits of large learning rates for kernel methods]{On the Benefits of Large Learning Rates for Kernel Methods}
\usepackage{times}
\coltauthor{%
 \Name{Gaspard Beugnot} \Email{gaspard.beugnot@inria.fr}\\
 \addr Inria, \'Ecole normale sup\'erieure, CNRS, PSL Research University, 75005 Paris, France
 \AND
 \Name{Julien Mairal} \Email{julien.mairal@inria.fr}\\
 \addr Univ. Grenoble Alpes, Inria, CNRS, Grenoble INP, LJK, 38000 Grenoble, France
 \AND
 \Name{Alessandro Rudi} \Email{alessandro.rudi@inria.fr}\\
 \addr Inria, \'Ecole normale sup\'erieure, CNRS, PSL Research University, 75005 Paris, France
}

\begin{document}
\maketitle

\begin{abstract}%
    This paper studies an intriguing phenomenon related to the good generalization performance of estimators obtained by using large learning rates within gradient descent algorithms.
    First observed in the deep learning literature, we show that
    such a phenomenon can be precisely characterized in the context of kernel methods, even though the resulting optimization problem is convex. Specifically, we consider the minimization of a quadratic objective in a separable Hilbert space, and show that with early stopping, the choice of learning rate influences the spectral decomposition of the obtained solution on the Hessian's eigenvectors. This extends an intuition described by~\citet{convex_annealing} on a two-dimensional toy problem to realistic learning scenarios such as kernel ridge regression. While large learning rates may be proven beneficial as soon as there is a mismatch between the train and test objectives, we further explain why it already occurs in classification tasks without assuming any particular mismatch between train and test data distributions.
\end{abstract}


\section{Introduction}

Gradient descent methods are omnipresent in machine learning, and a lot of effort has been devoted to better understand their theoretical properties. Optimal rates of convergence have been well characterized for minimizing convex functions in various contexts, including, for instance, stochastic optimization~\citep{nemirovski2009robust}.
For supervised learning, one is however more interested in the statistical optimality of the resulting estimator rather than in the ability to quickly optimize a training objective~\citep{bottou2007tradeoffs}. When considering both optimization and statistical questions, gradient descent methods were proven to be optimal under many assumptions \citep{yao2007early,pillaudvivien:hal-01799116}.

An important observation for this paper is that gradient descent algorithms typically require to tune some learning rate, or step size, to achieve the best performance. This has been thoroughly investigated in the optimization literature. For convex smooth problems in particular, the influence of step size on convergence rates is well understood~\citep{nesterov}. However, recent empirical studies have highlighted a surprising aspect of this parameter: when using gradient descent methods on neural networks, \emph{large} learning rates were found to be useful for obtaining \emph{good generalization properties}, or in other words, good statistical performance~\citep{pmlr-v139-jastrzebski21a}, even though they may be sub-optimal from an optimization point of view.

This paper aims at understanding this phenomenon from a broad but simple perspective, where both the function $\LEmp$ we optimize and the function $\LPop$ used to evaluate the statistical performance are quadratic forms of some separable Hilbert space $\HH$. Specifically, we assume that
\begin{equation}\label{eq:intro_model}
    \forall \theta \in \HH, ~~ \LEmp(\theta) = \frac{1}{2} \norm{\theta - \optEmp}_\TEmp^2 + \minLEmp, ~~~\text{and}~~~ \LPop(\theta) = \frac{1}{2} \norm{\theta - \optPop}_\TPop^2,
\end{equation}
where $\norm{\cdot}_A$ denotes the norm on $\HH$ induced by a positive definite operator $A$, \emph{i.e.}, $\norm{\theta}_A^2 = \dotprod{\theta}{A \theta}$ for any $\theta$ in $\HH$. With \cref{eq:intro_model}, $\LEmp$ and $\LPop$ are characterized by positive definite operators $\TEmp$ and $\TPop$, along with their minimizers denoted by $\optEmp$ and $\optPop$, respectively. The constant value $\minLEmp$ does not affect the optimization problem and can be safely ignored in the rest of this presentation.
The model from \cref{eq:intro_model} captures a large class of problems such as learning with kernels, detailed in \cref{sec:application_kernel_low_noise}, but we give here a simple example with ridge regression.

\begin{example}[$\TEmp$ and $\TPop$ with Ridge Regression.]
    Let $x_1, \dots, x_n$ be data points in $\RR^d$, and $y_1, \dots, y_n$ prediction variables, with $n \geq d$. Define $X \in \RR^{n \times d}$ the data matrix. We consider the ridge regression estimator with regularization $\lambda>0$, which is defined as the minimum of
    \begin{equation*}
        \forall \theta \in \RR^d, ~~ \LEmp(\theta) = \frac{1}{n} \sum_{i=1}^n \frac{1}{2} (\theta^\top x_i - y_i)^2 + \frac{\lambda}{2} \norm{\theta}^2 = \frac{1}{2n}\norm{X \theta - y}^2 + \frac{\lambda}{2} \norm{\theta}^2.
    \end{equation*}
    $\LEmp$ is a quadratic function of $\theta$, which can be rewritten as in \cref{eq:intro_model} with
    \begin{equation*}
        \forall \theta \in \RR^d, ~~ \LEmp(\theta) = \frac{1}{2} \norm{\theta - \optPop}_{\TEmp}^2 + \minLEmp, ~~ \text{with} ~~ \begin{cases}
            \optPop & = \frac{1}{n} \p{\frac{1}{n} X^\top X + \lambda \II_d}^{-1} X^\top y \\
            \TEmp   & = \frac{1}{n} (X^\top X + \lambda \II_d).                            \\
        \end{cases}
    \end{equation*}

    Assuming that the output can be written $y_i = x_i^\top \nu^\star + \epsilon_i$ with $\nu^\star \in \RR^d$ and $\epsilon_i$ some independent, zero-mean noise, then the \emph{population loss} is $\calP(\theta) = \EE \nicefrac12 \p{\theta^\top x_i - y_i}^2$, and the \emph{excess risk} defined by $\LPop(\theta) = \calP(\theta) - \inf_{\nu} \calP(\nu)$ is given with
    \begin{equation*}
        \LPop(\theta) = \EE \csb{\frac{1}{2} \p{(\theta-\nu^\star)^\top x}^2} = \frac{1}{2} \norm{\theta - \optPop}_\TPop^2, ~~ \text{with} ~~ \TPop = \EE \left[x x^\top\right].
    \end{equation*}
\end{example}
In this example, a discrepancy between train and test losses (between $\TEmp$ and $\TPop$) may occur in particular situations (\emph{e.g.}, presence of data augmentation during training, or simply mismatch between train and test distributions). The next example shows that such a mismatch may be in fact frequent for classification problems, even when train and test distributions do match.

\begin{example}[Discrepancy between train and test losses in classification with separable classes.]\label{ex:intro_classif}
    The scenario described in \cref{eq:intro_model} is particularly evident in the context of binary classification, when the classes are separable by a non-zero margin.
    This is considered a typical situation in many learning scenarios of interest, as classification over natural images---motivating the wide use of large-margin based classifiers in the field \citep{rawat2017deep}. We highlight here, that in this context, the loss we are using for training is not the best loss to consider for the test error, as discussed next.
    More precisely, consider a classification problem with two classes with non-zero margin. Let $\XX \subseteq \mathbb{R}^d$ and $\YY =\{-1,1\}$ be the input and the output space. Denote by $\rho(x,y) = \rho_\XX(x) \rho(y|x)$ the probability distribution describing the classification problem, where $\rho_\XX$ is the marginal probability over $\XX$, while $\rho(y|x)$ is the conditional probability of $y$ given $x$. The error that we would like to minimize is the binary error on the population, i.e. $B(\theta) = \mathbb{P}[ \sign[\theta(x)] \neq y]$ for a model $\theta$. Let $\nu^\star$ be a function minimizing the binary error and ${\cal H}$ be the class of models under consideration. Assume, for simplicity, that ${\cal H}$ is a RKHS with norm  $\|\cdot\|$ \citep{aronszajn50reproducing}, \emph{i.e.}, there exists a map $\phi:\XX \to {\cal H}$ such that the function in ${\cal H}$ are characterized as $\theta(x) = \langle \phi(x), \theta\rangle$, for any $\theta$ in ${\cal H}$.
    It has been shown by \citet{pillaudvivien:hal-01799116} (in particular, Lemma 1 and Appendix A, Theorem 13), that in the context of two classes separated by a non-zero margin and whose conditional probability is regular enough, then $\nu^\star$ is in $\HH$ and moreover $B(\theta) - B(\nu^\star) \leq e^{-c/\|\theta - \nu^\star\|}$ for some constant $c$.
    Therefore, the binary error decreases exponentially in terms of the Hilbert norm $\|\cdot\|$. On the other hand, the norm minimized at training time is some smooth convex surrogate of the binary loss, as, for example, the quadratic loss.

    In this case, the trained vector may be obtained by minimizing the population loss (in fact, a regularized empirical version, but we omit this fact here for simplicity) such that $F(\theta) = \mathbb{E} (\theta(x) - y)^2$.
    Noting that $\int (\theta(x) - \optEmp(x))^2 \dd\rho(x) = \|\TEmp^{1/2}(\theta - \optEmp)\|^2 = \|\theta - \optEmp\|_\TEmp$ for $\TEmp = \int \phi(x) \phi(x)^\top \dd \rho_\XX(x)$, and $\optEmp = \nu^\star$ under the considered conditions (see the same paper), we have
    \begin{equation*}
        F(\theta) - F(\optEmp) = \int (\theta(x) - \optEmp(x))^2 \dd\rho_\XX(x) = \|\theta - \optEmp\|^2_{\TEmp}.
    \end{equation*}
    This is a typical case, where there is a discrepancy between the error of interest
    \begin{equation*}
        R(\theta) - R(\nu^\star) = \|\theta - \nu^\star\|^2,
    \end{equation*}
    which, if optimized, would lead to an exponential decrease of the classification error, and the loss $F$ that is instead optimized by the algorithm at training time, for which we have only the slower rate $B(\theta) - B(\nu^\star) \leq (F(\theta) - F(\optPop))^\alpha$, with $\alpha \in (1/2, 1)$ (see, e.g. \citet{audibert2004classification} or \citet{audibert2007fast} for what concerns the CAR assumption).
\end{example}

In this paper, we are interested in understanding in which regime large learning rates with early stopping could be useful for kernel methods, even if they are suboptimal from an optimization point of view.
We consider indeed the optimization of $\LEmp$ in \cref{eq:intro_model} with plain gradient descent, starting from a vector $\theta_0$ in $\HH$ with step-size $\eta$, and
we distinguish between two cases: having a \emph{small} learning rate $\eta_s$ or a \emph{large} learning rate $\eta_b$, the range of both is to be detailed later.
A simple intuition was suggested by~\citet{convex_annealing} on a two-dimensional toy problem, showing that large learning rates may be beneficial as soon as there is a mismatch between $F$ and $R$ (meaning, what we train on does not correspond to what we test on). We show that such an insight can be extended beyond toy problems to realistic scenarios with traditional kernel methods, and that, perhaps surprisingly, this phenomenon occurs already in simple classification tasks.


\begin{theorem}[Informal version of our main result]\label{thm:informal_intro}
    Under a few assumptions described later in this paper, consider the target accuracy~$\alpha$ and large and small step sizes $\theta_b$ and $\theta_s$ (these quantities being defined in the aforementioned assumptions).
    Consider the gradient descent iterations $ \theta_{t+1} = \theta_{t} - \eta \TEmp (\theta_{t} - \optEmp)$ either with step size $\eta=\eta_b$ or $\eta=\eta_s$, and stop the procedure as soon as $F(\theta_{t+1}) \leq \alpha$, resulting in two estimators $\theta_b$ or $\theta_s$. Then,
    \begin{equation}\label{eq:informal_thm_bound}
        \LPop(\theta_b) - \LPop(\optPop) ~\leq~ 34 \frac{\kappa_{\TPop}}{\kappa_\TEmp} (\LPop(\theta_s)- \LPop(\optPop) ),
    \end{equation}
    where $\kappa_{\TPop}$ and $\kappa_{\TEmp}$ are the condition numbers of the operators $\TPop$ and $\TEmp$, respectively, restricted to~$\HH_n$.
\end{theorem}
Note that $\LPop(\optPop) = 0$ by definition of $R$; we made this quantity explicit in the bound for clarity purposes. The main conclusion from the theorem is that with early stopping (and with a target accuracy that is reasonable according to statistical learning theory, as discussed later), large learning rates can provide better estimators than small ones, even though the quantity $\eta_s$ may yield much faster convergence for minimizing the objective function $F$ than $\eta_b$ (a fact also discussed later).
This phenomenon occurs when the condition number $\kappa_\TEmp$ is much larger than $\kappa_\TPop$, which may already arise in classification tasks, as mentioned earlier in Example~\ref{ex:intro_classif} where $\kappa_\TPop = 1$.

Note that \cref{eq:informal_thm_bound} raises several questions and could be easily misinterpreted, since such a relation may suggest that an arbitrarily small risk $R(\theta_b)$ could be obtained by considering minimization problems that are arbitrarily badly conditioned. Unfortunately, but not surprisingly, there is however no free lunch here, as discussed in the next remark.

\begin{remark}[The issue of ill-conditioning.]
    A naive observation is that $R(\theta_b)$ could be arbitrarily small by making the problem more ill-conditioned. However, the bound on $\LPop(\theta_b)$ in \cref{eq:informal_thm_bound} is relative to $\LPop(\theta_s)$. Notably, a careful reading of the proof shows that $\LPop(\theta_s)$ is an increasing function of the conditioning number, for the chosen level sets $\alpha$.
\end{remark}

\paragraph{Summary of contributions.}
Our first contribution is the relation described in \cref{eq:informal_thm_bound}, highlighting potential benefits of large learning rate strategies when the training objective has a worse condition number than the one used to evaluate the quality of the estimator. This is illustrated in \cref{fig:estimators_path}, a figure inspired by~\citet{convex_annealing}.
Our second contribution is to show that such a mismatch systematically occurs in simple classification scenarios with low noise, where the quantity of interest to minimize may not be the population risk, as discussed earlier.
Overall, this allows us to show that the previous phenomenon occurs in realistic learning scenarios with kernels, which we also check in practice through numerical experiments.


\begin{figure}[h]
    \centering
    \begin{tikzpicture}
        \pgfmathsetmacro{\sa}{1.}
        \pgfmathsetmacro{\sb}{.1}
        \pgfmathsetmacro{\saR}{.8}
        \pgfmathsetmacro{\sbR}{.4}
        \pgfmathsetmacro{\levelseta}{.5}
        \pgfmathsetmacro{\levelsetb}{1.}
        \pgfmathsetmacro{\levelsetc}{2.}
        \pgfmathsetmacro{\thetaAX}{-4.}
        \pgfmathsetmacro{\thetaAY}{2.}

        \draw[->] (-5,0) -- (5,0) node[below right] {$e_2$};
        \draw[->] (0,-2.5) -- (0,2.5) node[left] {$e_1$};
        \filldraw[black] (0,0) circle (1pt) node[below left] {$\optEmp$};

        \pgfmathsetmacro{\thetaOptX}{sqrt(\levelseta/\sa)/4}
        \draw[dashed,rotate around={60:(\thetaOptX,\thetaOptX)}] (\thetaOptX,\thetaOptX)
        ellipse ({sqrt(\levelseta/\sbR)} and {sqrt(\levelseta/\saR)});
        \draw[dashed,rotate around={60:(\thetaOptX,\thetaOptX)}] (\thetaOptX,\thetaOptX)
        ellipse ({sqrt(\levelsetb/\sbR)} and {sqrt(\levelsetb/\saR)});
        \filldraw[black] (\thetaOptX,\thetaOptX) circle (1pt) node[above right] {$\optPop$};

        \draw[red] (0, 0) ellipse ({sqrt(\levelseta/\sb)} and {sqrt(\levelseta/\sa)});
        \draw (0, 0) ellipse ({sqrt(\levelsetb/\sb)} and {sqrt(\levelsetb/\sa)});
        \draw (0, 0) ellipse ({sqrt(\levelsetc/\sb)} and {sqrt(\levelsetc/\sa)});

        \filldraw[black] (\thetaAX, \thetaAY) circle (1pt) node[above left] {$\theta_0$};

        \pgfmathsetmacro{\etaSmall}{0.2}
        \foreach[evaluate={
                    \thetaX = (1 - \etaSmall * \sb)^\t * \thetaAX;
                    \thetaY = (1 - \etaSmall * \sa)^\t * \thetaAY;
                    \thetaXPrev = (1 - \etaSmall * \sb)^(\t - 1) * \thetaAX;
                    \thetaYPrev = (1 - \etaSmall * \sa)^(\t - 1) * \thetaAY;
                }] \t in {1,...,30}{
                \filldraw[blue] (\thetaX, \thetaY) circle (1pt);
                \draw[blue, dotted] (\thetaX, \thetaY) -- (\thetaXPrev, \thetaYPrev);
            }
        \pgfmathsetmacro{\etaBig}{1.95}
        \foreach[evaluate={
                    \thetaX = (1 - \etaBig * \sb)^\t * \thetaAX;
                    \thetaY = (1 - \etaBig * \sa)^\t * \thetaAY;
                    \thetaXPrev = (1 - \etaBig * \sb)^(\t - 1) * \thetaAX;
                    \thetaYPrev = (1 - \etaBig * \sa)^(\t - 1) * \thetaAY;
                }] \t in {1,...,19}{
                \filldraw[brown] (\thetaX, \thetaY) circle (1pt);
                \draw[brown, dotted] (\thetaX, \thetaY) -- (\thetaXPrev, \thetaYPrev);
            }
    \end{tikzpicture}
    \vspace{-1ex}
    \caption{We optimize the quadratic $\LEmp$ \emph{(level sets are filled lines, centered in $\optEmp$)} with gradient descent, starting from $\theta_0$ until we reach the level sets $\alpha$ \emph{(filled line, red)}. However, we evaluate the quality of the estimator through $\LPop$ \emph{(level sets are dashed lines, centered in $\optPop$)}. Doing small step size \emph{(blue dots)} optimizes the direction $e_1$ first, and yields an estimate which is far from $\optPop$ in $\TPop$ norm; doing big step size \emph{(brown dots)} oscillates in the direction $e_1$, but ultimately yields an estimator which is close to $\optPop$ in $\TPop$ norm.}
    \label{fig:estimators_path}
\end{figure}
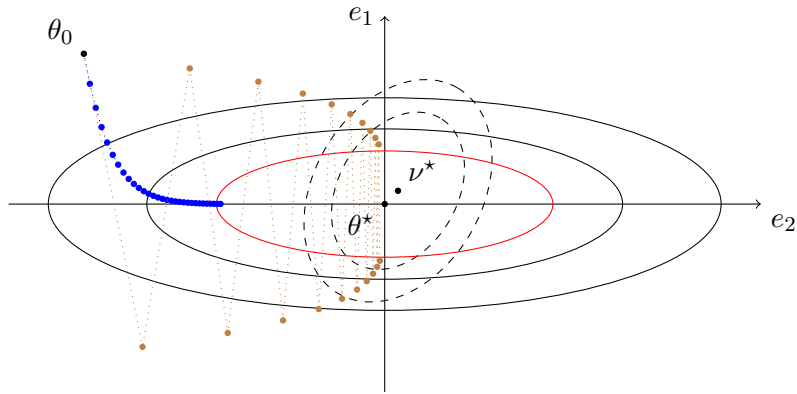

\section{Related Work}
Our main motivation is to better understand the role of learning rate in obtaining good generalization for supervised learning. Even though empirical benefits of large learning rates were first described for neural networks, a few recent works have studied this phenomenon for convex problems. We review here some relevant work.

\paragraph{Setting the learning rate in neural networks.}
Stochastic gradient descent has become the standard tool for optimizing neural networks. When the learning rate is very small, the network evolves in a so-called ``lazy-regime'' where its dynamics are well understood \citep{NEURIPS2019_ae614c55,NEURIPS2018_5a4be1fa} but which fails to capture the good generalization performance observed with large learning rate.
Specifically, this phenomenon has been empirically observed numerous times \citep[see, for instance][]{smith2021on,pmlr-v139-jastrzebski21a,Jastrzebski2020The}; common strategies consist of using first a large learning rate, before annealing it to a smaller value. As a first step towards proving theoretically the effect of choosing large learning rates for training neural networks, \citet{NEURIPS2019_bce9abf2} devise a two-layer neural network model with different set of features where the order in which they are learnt matters, where the previous annealing strategy could be shown to be useful in theory.

\paragraph{A convex perspective.}
Recently, different papers tried to reproduce this phenomenon in convex settings. This is probably thanks to the observation made by \cite{convex_annealing}, where a toy dataset is exhibited, which was the main motivation for this work. However, it fails to capture realistic scenarios where the data distribution is not isotropic, or with non linear data embeddings. \citet{wu2021direction} aimed at filling these gaps but again, relies on the data distribution to be linear, isotropic, with the number of dimension going to infinity in order to have all data points approximately orthogonal; we do not make any of those assumptions. The bound with the condition number we obtain in \cref{thm:main_body} is totally new. Finally, we highlight that we use \emph{plain} gradient descent, and do not need stochasticity to exhibit the big learning rate phenomenon. This is consistent with recent work \citep{DBLP:journals/corr/abs-2109-14119} which shows that SGD is not necessary to obtain state of the art performances, and that GD simply needs a better fine tuning of hyper parameters.


\section{Main Result}

In this section, we show that by performing standard gradient descent on the empirical loss $\LEmp$, choosing a big learning rate will first optimize the smallest eigencomponent of $\TEmp$. That is, the resulting estimator is mostly located on the biggest eigenvector of $\TEmp$. On the other hand, the smaller the learning rate, the more will the solution be located on the small eigencomponents, with biggest eigenvectors of $\TEmp$ being learnt first.

\subsection{Settings and notations}

\paragraph{Gradient descent updates.}
We perform standard gradient descent on the empirical loss $\LEmp$, starting from some $\theta_{0} \in \HH$, with step size $\eta$. We obtain
\begin{equation}\label{eq:gradient_descent_update}
    \forall t \geq 0, ~~ \theta_{t+1} = \theta_{t} - \eta \TEmp (\theta_{t} - \optEmp), ~~ \text{thus} ~~ \theta_{t} - \optEmp = (\II - \eta \TEmp)^t (\theta_{0} - \optEmp).
\end{equation}
This enables a very simple analysis of the training in the eigenbasis of $\TEmp$. We now give a more precise definition of the model in \cref{eq:intro_model}.

\begin{assumption}[Representer theorem assumption]\label{asmpt:notation_operator}
    There is a $n$-dimensional subspace $\HH_n \subseteq \HH$ that is invariant by~$T$---that is, $T \theta$ is in $\HH_n$ for all $\theta$ in $\HH_n$ and such that $\optPop$ is in $\HH_n$.

    We denote by $(\sigma_i, e_i)$ the eigenbasis of the p.d. operator $\TEmp$ restricted to $\HH_n$, with $\sigma_1 > \dots > \sigma_n > 0$, assuming eigenvalues are distinct from each other, and we call $\kappa_\TEmp = \sigma_1/\sigma_n$ the condition number.
    Similarly, the restriction of $\TPop$ to $\HH_n$ is a positive definite operator whose spectrum is $\varsigma_1 > \dots > \varsigma_n$, with condition number $\kappa_\TPop = \varsigma_1 / \varsigma_n$. Since, rescaling the objectives $F$ and $R$ by constant factors does not change their minimizers, we also safely assume that $\sigma_1 = \varsigma_1 = 1$.
\end{assumption}
The model described by~\cref{asmpt:notation_operator} is quite natural, and ensures that a representer theorem holds when learning on a finite training set of $n$ points. It is notably satisfied in classical learning formulations with kernels.

With the notations of \cref{asmpt:notation_operator}, we can now rewrite the update of \cref{eq:gradient_descent_update} along a specific direction $e_i$:
\begin{equation}\label{eq:gradient_descent_update_eigen}
    \forall t \geq 0, ~~ \dotprod{\theta_{t} - \optEmp}{e_i}_\HH = (1 - \eta \sigma_i)^t \dotprod{\theta_{0} - \optEmp}{e_i}_\HH.
\end{equation}
Consider the quantity $\absv{1 - \eta \sigma_i}$. The closer to $0$, the smaller will the $i$-th component of $\theta_t - \optEmp$ on the eigenbasis be when the number of steps $t$ increases. We plot $\absv{1 - \eta \sigma_i}$ in \cref{fig:attenuation_coefficient_function_lr}.

\input{figures/attenuation_coeff}

Specifically, two ranges of learning rate naturally appear. With a \emph{small} learning rate satisfying $\eta_s < 2/(\sigma_1 + \sigma_n)$, we see on \cref{fig:attenuation_coefficient_function_lr} that the attenuation $\absv{1 - \eta \sigma_i}$ is biggest for the smallest eigenvalue. On the other hand, with a \emph{big} learning rate satisfying $\eta_b > 2/(\sigma_1 + \sigma_n)$, we see that the attenuation is biggest for the biggest eigenvalue. This motivates the next assumption.

\begin{assumption}[Learning rate]\label{asmpt:lr}
    The learning rates satisfy
    \begin{equation}
        0 < \eta_s < \frac{2}{\sigma_1 + \sigma_n} < \eta_b < \frac{2}{\sigma_1}.
    \end{equation}
\end{assumption}

Note that the quantity $\eta = \frac{2}{\sigma_1 + \sigma_n}$ which naturally appears in \cref{fig:attenuation_coefficient_function_lr} is a classical upper bound for proving the convergence of $\sigma_1$-smooth and $\sigma_n$-strongly convex function, of which $\LEmp$ belongs to, see \emph{e.g.} Thm 2.1.15 in \cite{nesterov}. The rate $1/\sigma_1$ is the classical one when we do not have a strong convexity assumption.
This means that in our model, the concept of ``small'' learning rate simply means being of the order of the best possible learning rates available from an optimization point of view, while the concept of ``large'' means being close to values leading to diverging algorithms.

\begin{remark}[Biggest learning rate before divergence.]
    In a recent work, \citet{cohen2021gradient} observed that neural networks trained with gradient descent and ``good'' constant step size $\eta$ were often in a regime where $\sigma_1$ -- the maximum value of the Hessian of the loss -- hovered just around $2/\eta$. It is surprisingly analogous to our model: the range of learning rate we consider for big step sizes in Assumption~\ref{asmpt:lr} enforces $2/\eta_b$ to be close to $\sigma_1$.
\end{remark}

\begin{remark}[Oscillating weights.]
    Geometrically, having $\eta > 1/ \sigma_1$ means that the estimator will oscillate along the direction $e_1$, \textit{i.e.} $\dotprod{\theta_{t} - \optEmp}{e_1}$ will change sign at each iteration. Such behaviour was observed for neural networks trained with classical learning rate strategies, where the weights' sign change in the early phase of training \citep{xing2018walk}.
\end{remark}

Then, the following technical assumption is needed to ensure that there is a signal on the lowest and biggest eigendirection. It is satisfied \emph{e.g.} as soon as the initialization is chosen at random.

\begin{assumption}[Initialization]\label{asmpt:init_iota}
    We assume
    \begin{align*}
        \dotprod{\theta_0 - \optEmp}{e_1}_\HH \neq 0, ~~ \dotprod{\theta_0 - \optEmp}{e_n}_\HH \neq 0.
    \end{align*}
\end{assumption}

Finally, we assume that the target accuracy in terms of optimization is not too small compared to the model error $R(\optPop)$:
\begin{assumption}[Target accuracy and model error]\label{asmpt:target}
    Consider some learning rates $\eta_b$ and $\eta_s$ chosen in the range of Assumption~\ref{asmpt:lr}. We assume that the target accuracy $\alpha$ satisfies $\alpha \leq \alpha_1$, where $\alpha_1$ is given in Definition~\ref{def:alpha_1} and only depends on the spectrum of $\TEmp$ and the learning rates. Furthermore, we assume that
    \begin{equation}\label{eq:asmpt_target_bound}
        \frac{\LPop(\optEmp)}{\alpha} \leq \min \cb{\frac{1}{4}, \frac{\kappa_\TEmp}{72\kappa_\TPop}}.
    \end{equation}
\end{assumption}
This assumption is twofold, providing on the target accuracy $\alpha$ both an \emph{upper} bound (with $\alpha \leq \alpha_1$) and a \emph{lower} bound (with \cref{eq:asmpt_target_bound}). The \emph{lower bound} not being satisfied amounts to $F$ being a poor approximation of $R$: $R(\theta^\star)$ is too big, or the two ellipsoids' centers are far apart in \cref{fig:estimators_path}. Usual machine learning settings often involve large amounts of data, where the limiting factor for good generalization is poor optimization rather than lacking information, so we can expect this assumption to often hold in practice.
The \emph{upper bound} stems from the proof technique of our main result in Theorem~\ref{thm:main_body}, a brief sketch of which is available in \cref{sec:sketch_of_proof}. The proof relies on the fact that sufficiently many steps $t_s$ (resp. $t_b$) are made before the gradient descent is stopped, so that the biggest (resp. the smallest) eigen component is attenuated enough. To ensure this, we can \emph{(i)} either make the learning rate smaller\footnote{Informally, by making $\eta_b \to 2/\sigma_1$ (resp. $\eta_s \to 0$) we need bigger $t_b$ (resp. $t_s$) to achieve optimization error $\alpha$. Thus, $\alpha < \alpha_1$ can be replaced with $\eta_b > 2/\sigma_1 - \epsilon$ (resp. $\eta_s < \epsilon$).}, or \emph{(ii)} take the target error $\alpha$ sufficiently small. We choose the latter, hence the assumption $\alpha \leq \alpha_1$.



\subsection{The Main Theorem}

We now give our main theorem, whose proof is given in \cref{sec:appendix_proof_main_thm}.
\begin{theorem}[Benefits of large learning rates]\label{thm:main_body}
    Consider the different quantities defined in Assumptions~\ref{asmpt:notation_operator}, \ref{asmpt:lr}, \ref{asmpt:init_iota} and~\ref{asmpt:target}.
    Then, perform the gradient descent updates of~\cref{eq:gradient_descent_update}, with either small step size $\eta_s$ or big step size $\eta_b$, and stop as soon as $F(\theta_{t}) \leq \alpha$, assuming that the resulting estimators satisfy $F(\theta_s) \geq \alpha/2$ and $F(\theta_b) \geq \alpha/2$\footnote{This is a mild assumption that could be removed at the price of cumbersome technical details.}.
    Then,
    \begin{equation}\label{eq:thm_body_main_bound}
        \LPop(\theta_b) - \LPop(\nu^\star) ~\leq~ 34 \frac{\kappa_{\TPop}}{\kappa_\TEmp} (\LPop(\theta_s)-\LPop(\nu^\star)). 
    \end{equation}
\end{theorem}

Recall that $\LPop$ has minimum $0$, so \cref{eq:thm_body_main_bound} essentially guarantees better performance of $\theta_b$.
The estimator obtained with big step size is better that the one obtained with small step size as soon as the operator $\TEmp$ is ill-conditioned. This is notably the case when doing classification with kernel methods with the ridge estimator, as we discuss in \cref{sec:application_kernel_low_noise}.

\subsection{Discussion}

\paragraph{Implications in classification with separable classes.} In the context of the example discussed in the introduction, we see clearly that, under the simplifying hypothesis that the population error behaves similarly to the empirical error, the choice of the learning step has the unexpected impact of reducing the Hilbert norm by a multiplicative constant that can be significantly smaller than $1$, leading to an exponential improvement in the classification error.


\paragraph{Comparison with analysis techniques based on learning rate annealing.}
Most recent approa\-ches to explain the role of the learning rate in the generalization \citep{NEURIPS2019_bce9abf2,wu2021direction} rely on \emph{annealing} the learning rate: the first phase of the training is carried out with a large step size, before it is discounted to a lower value. We do not need such mechanism in our theoretical analysis, which turns to be simpler with a unique value for the step size. With annealing, our analysis could sum up the following way: do $t$ steps with learning rate greater or equal to $2/\sigma_1$, so that all attenuation coefficient $\absv{1 - \eta_b \sigma_i}$ in \cref{eq:gradient_descent_update_eigen} are smaller than $1$ except for the first (few) eigencomponent. Doing so, all eigendirections would be optimized except for the first (few) ones. Then, anneal the learning rate until the $\alpha$ level set of the loss are reached.

\paragraph{Discussion on the complexity.}
The fact that the result of Theorem~\ref{thm:main_body} relies on the condition number can be somewhat surprising. Indeed, we may wonder why the other eigenvalues of the spectrum do not play a role in the result. This is in fact due to the proof technique, which relies on comparing the estimator mostly located on $e_1$ (for big learning rates) and the one mostly on $e_n$ (for the small learning rates). Thus, the distance $\sigma_{n-1} - \sigma_n$ and $\sigma_1 - \sigma_2$ play a role in the \emph{complexity} of the gradient descent, which is highlighted by the next lemma, which is a consequence of Lemmas~\ref{lem:technical_bound_big_lr} and \ref{lem:technical_bound_small_lr} in \cref{sec:appendix_proof_main_thm}.

\begin{lemma}[Computational complexity]\label{lem:computational_complexity_ts_tb}
    Under the settings of Theorem~\ref{thm:main_body}, denote $t_s$ (resp. $t_b)$ the number of steps necessary to obtain $\theta_s$ (resp. $\theta_b$). Then,
    \begin{equation}
        t_s \geq O\p{\log \frac{1 - \eta_s\sigma_{n}}{1 - \eta_s\sigma_{n-1}}},
        ~~~
        t_b \geq O\p{\log \frac{\absv{\eta_b \sigma_1 - 1}}{\max \cb{\eta_b \sigma_2 - 1, 1 - \eta_b \sigma_n}}}.
    \end{equation}
\end{lemma}
In particular, if $s= \sigma_{n-1} - \sigma_n$, then $t_s = O(1/s)$, and the same holds for $t_b$ with $s =\sigma_1 - \sigma_2$.

We emphasize that the complexity incurred by using learning rates satisfying Assumption~\ref{asmpt:lr} may be very large. However, the motivation of our work is to study an existing practice in deep learning (using large step sizes that are not optimal in terms of optimization of the training loss, but which are better in terms of test loss). We study this phenomenon from a kernel perspective to be able to leverage theoretical tools. Analysing the computational complexity to obtain a given statistical accuracy in convex problems is a separate and well-documented problem, and we do not necessarily advocate the use of very large step sizes for kernel regression.

\paragraph{Nystr\"om projections.}
In practice projections techniques (known as Nyström projections) are used to reduce the dimension and avoid a cost quadratic in $n$ for gradient descent. Given some $m \ll n$, this amounts to choosing $m$ anchor points among the data and approximating the kernel matrix $K$ with a rank-$m$ matrix $\tilde{K} = K_{nm}K_m^{-1}K_{nm}^\top$. Fortunately, this is still encompassed in the framework of the model in \cref{eq:intro_model}. Indeed, the resulting problem is still a quadratic problem, for which a representer theorem holds -- the solution lies on the span of the anchor points.

\paragraph{Beyond the square loss.}
The result of \cref{thm:main_body} relies on using the square loss for obtaining a closed-form expression of the gradient descent update. This is fairly common in learning theory and already provides interesting hindsight. A natural extension to other loss functions would be to consider local quadratic approximations using the Hessian. This is for instance what has been done for kernel ridge regression and self-concordant loss functions by \cite{10.1214/09-EJS521}.

\subsection{Sketch of proof for the main result}
\label{sec:sketch_of_proof}
The detailed proof is delayed to \cref{sec:appendix_proof_main_thm}. The idea is the following: by tuning the number of steps, we can have the estimator trained with small (resp. big) step size mostly aligned with the smallest (resp. biggest) eigenvector.

\paragraph{The directional bias induced by the step size.}
As shown in \cref{fig:attenuation_coefficient_function_lr}, having the learning rates satisfying Assumption~\ref{asmpt:lr} ensures that the quantities
\begin{equation}
    \epsilon_b^2 =
    \frac{\sum_{2 \leq i \leq n} \dotprod{\theta_b - \optEmp}{e_i}^2}{\dotprod{\theta_b - \optEmp}{e_1}^2},
    ~~
    \epsilon_s^2 =
    \frac{\sum_{1 \leq i \leq n-1} \dotprod{\theta_s - \optEmp}{e_i}^2}{\dotprod{\theta_s - \optEmp}{e_n}^2},
\end{equation}
can be made arbitrarily small, while Assumption~\ref{asmpt:init_iota} ensures they are well defined. $\epsilon_b$ (resp. $\epsilon_s$) quantifies to what extent is $\theta_b - \optEmp$ (resp. $\theta_s - \optEmp$) mostly on $e_1$ (resp. $e_n$). For instance, in the extreme case where $\epsilon_b = 0$ (resp. $\epsilon_s = 0$), then $\theta_b - \optEmp = x e_1, x\in \RR$ (resp. $\theta_s - \optEmp = y e_n, y \in \RR$). To see why it can be made small, refer to the closed-form expression of $\theta_{(\eta, t)}$ in \cref{eq:gradient_descent_update,eq:gradient_descent_update_eigen}, and assume for simplification that $\dotprod{\theta_0 - \optEmp}{e_i} = c_0$ for all $i$, with $c_0$ some constant factor. This gives
\begin{equation}
    \epsilon_b^2 =
    \frac{\sum_{2 \leq i \leq n} (1 - \eta_b \sigma_i)^{2t}}{(1-\eta_b \sigma_1)^{2t}},
    ~~
    \epsilon_s^2 =
    \frac{\sum_{1 \leq i \leq n-1} (1 - \eta_s \sigma_i)^{2t}}{(1-\eta_s \sigma_n)^{2t}}.
\end{equation}
Following the discussion of Assumption~\ref{asmpt:lr}, we have that $|1-\eta_b \sigma_1| > |1 - \eta_b \sigma_i|$ for all $i > 1$, and $|1 - \eta_s \sigma_n| > |1 - \eta_s \sigma_n|$ for all $i < n$. Thus, we have that for any $\delta > 0$,
\begin{equation}
    t_b \geq \frac{1}{2} \frac{\log 1/\delta^2}{\log \frac{\eta_b \sigma_1 - 1}{\max_{2 \leq i \leq n} |1 - \eta_b \sigma_i|}} \implies \epsilon_b^2 \leq \delta^2,
    ~~
    t_s \geq \frac{1}{2} \frac{\log 1/\delta^2}{\log \frac{1 - \eta_s \sigma_n}{\max_{1 \leq i \leq n-1} |1 - \eta_s \sigma_i|}} \implies \epsilon_s^2 \leq \delta^2.
\end{equation}

\paragraph{Different risk on the $\alpha$-level sets.}
In this paragraph, assume \emph{(A)} $\theta_b = x e_1$, and $\theta_s = y e_n$, that is $\epsilon_b = \epsilon_s = 0$, \emph{(B)} that $\LPop(\optEmp) = 0$ and \emph{(C)} that
\begin{equation}\label{eq:sketch_proof_level_sets}
    \alpha/2 \leq \LEmp(\theta_b) \leq \alpha, ~~ \alpha/2 \leq \LEmp(\theta_s) \leq \alpha.
\end{equation}
Then, we have for $\theta_b$ that
\begin{equation}\label{eq:sketch_proof_big_lr}
    \LPop(\theta_b) = \frac{1}{2} \norm{x e_1}_\TPop^2 \leq \frac{1}{2} \varsigma_1 x^2 \leq \frac{\alpha}{2} \frac{\varsigma_1}{\sigma_1},
\end{equation}
where we used \cref{eq:sketch_proof_level_sets} to bound $\alpha \geq \LEmp(\theta_b) = 1/2 \cdot \sigma_1 x^2$. We do the same with $\theta_s$, this time using $1/2 \cdot \sigma_n y^2 = \LEmp(\theta_s) \leq \alpha$ to obtain
\begin{equation}\label{eq:sketch_proof_small_lr}
    \LPop(\theta_s) = \frac{1}{2} \norm{y e_n}_\TPop^2 \geq \frac{1}{2} \varsigma_n y^2 \geq \alpha \frac{\sigma_n}{\varsigma_n},
\end{equation}
Finally, combining \cref{eq:sketch_proof_big_lr} with \cref{eq:sketch_proof_small_lr}, we obtain
\begin{equation*}
    \LPop(\theta_b) \leq 2 \frac{\kappa_\TPop}{\kappa_\TEmp} \LPop(\theta_s).
\end{equation*}

\paragraph{Ensuring both conditions can be met together.}
We now point out the main differences between this simplified sketch of proof and the rigorous proof in \cref{sec:appendix_proof_main_thm}. First of all, we do not have \emph{(A)} but rather an approximation of it, with $\epsilon_s \leq \delta$ and $\epsilon_b \leq \delta$. Second, we do not have \emph{(B)} and rather take into account the error $\LPop(\optEmp)$ to derive Theorem~\ref{thm:main_body}. Finally, and most importantly, we check that we can have \emph{low $\epsilon$ and $F(\theta) \geq \alpha/2$ at the same time}. Indeed, we need a big number of iterations to achieve low $\epsilon_b$ or $\epsilon_s$. This implies better optimization of the objective function $\LEmp$. To prevent this, we can either tune the learning rate (having $\eta_s$ close to $0$ and $\eta_b$ close to $2 / \sigma_1$) or provide an upper bound on $\alpha$. We choose the later, hence the hypothesis $\alpha \leq \alpha_1$ in Assumption~\ref{asmpt:target}.

\section{Comparison with Results in Kernel Regression}
\label{sec:application_kernel_low_noise}

To provide some intuition over the result of Theorem~\ref{thm:main_body}, we consider its implications in a supervised learning setting, specifically classification on a low-noise dataset.

\subsection{Background on the kernel ridge regression estimator}
We consider standard settings. We denote by $\XX \subseteq \RR^d$ the input space, and $\YY = \cb{-1, 1}$ the output space. We draw $n$ i.i.d samples $(x_i, y_i)_{1 \leq i \leq n}$ from an unknown distribution $\rho$ on $\XY$, and we search a prediction function $\theta \in \HH$, where $\HH$ is a RKHS with feature map $\phi: \XX \to \HH$. We assume the kernel to be bounded by a constant $C_K$. See \cite{aronszajn50reproducing} for a precise account on RKHS. We use the square loss as loss function. In order to find a function $\theta$ which maps elements of $\XX$ to $\YY$, we optimize the (regularized) \emph{empirical risk} $\LEmp$, defined for all $\theta \in \HH$ and $\lambda \geq 0$ a regularization parameter with
\begin{equation}\label{eq:def_empirical_risk}
    \LEmp\p{\theta} = \frac{1}{n}\sum_{i=1}^n \frac{1}{2} \p{\theta(x_i) - y_i}^2 + \frac{\lambda}{2} \norm{\theta}^2.
\end{equation}
The minimizer $\optEmp$ of $\LEmp$ is always well defined as soon the training samples $x_i$ are distinct (in the case $\lambda = 0$), which we assume now. We will be minimizing $\LEmp$ with gradient descent when we are in fact interested in minimizing the \emph{test error}
\begin{equation}\label{eq:def_test_error}
    B(\theta) = {\mathbb P}[\sign[\theta(x) \neq y]].
\end{equation}

We will relate the test error and the empirical risk to quadratics forms in $\HH$ by means of other quantities. To do that, we first define the population loss along with its regularized version with
\begin{equation}\label{eq:def_pop_loss}
    \forall \theta, ~~ \mathcal{P} (\theta) = \int_{\XY} \frac{1}{2} \p{\theta(x) - y}^2 \dd\rho(x,y), ~~ \mathcal{P}_\lambda(\theta) = \mathcal{P}(\theta) + \frac{\lambda}{2} \norm{\theta}^2.
\end{equation}
The minimizer of $\mathcal{P}$ on $\LL_2(\rho_x)$ is the regression function $\regFunc(x) = \EE[y|x]$. It is an element of $\LL_2(\rho_x)$ but not necessary of $\HH$. We denote by $\optPop$ the minimizer of $\mathcal{P}_\lambda$ on $\HH$. If $\lambda>0$, it is always well defined; otherwise, with $\inclOp : \HH \to \LL_2(\rho_x)$ the inclusion operator, $\optPop$ exists as soon as the projection of the regression function on the closure of the range of $\inclOp$ belongs to the range of $\inclOp$. See \cite{JMLR:v6:devito05a} for a precise account.

In the following, we assume $\lambda \geq 0$ and that $\optPop$ is well defined, and we consider specific assumptions on $\rho$, \textit{via} assumptions on $\optPop$ and $\regFunc$.

\subsection{Relating supervised learning with quadratic forms in \texorpdfstring{$\HH$}{H}}

To relate the problems of \cref{eq:def_empirical_risk,eq:def_test_error} to quadratic forms in the RKHS, we simply need to introduce the \emph{empirical covariance operator}, with
\begin{equation}\label{eq:def_empirical_covariance}
    \TEmp = \frac{1}{n} \sum_{i=1}^n \phi(x_i) \otimes \phi(x_i).
\end{equation}
Then, as optimizing the Hilbert norm is a good proxy for optimizing the test error (following Example~\ref{ex:intro_classif} and Lemma~\ref{lem:small_hilbert_norm_statistically_optimal} in \cref{sec:appendix_low_noise_loucas_results}), we define
\begin{equation}\label{eq:def_empirical_risk_quadratic}
    \LEmp\p{\theta} = \frac{1}{2} \norm{\theta - \optEmp}_{\TEmp}^2 + \minLEmp, ~~ \LPop(\theta) = \frac{1}{2} \norm{\theta - \optPop}_\HH^2, ~~ \text{with} ~~ \minLEmp =  \frac{1}{2n} y^\top \csb{\II_n - \frac{K}{n}\p{\frac{K}{n} + \lambda}^{-1}}  y.
\end{equation}
$K$ is the kernel matrix ($K/n$ shares the same spectrum than $\TEmp$), and the minimum $\LEmp(\optEmp) = \minLEmp$ is $0$ when $\lambda = 0$. We are in the settings of the model of \cref{eq:intro_model}, and we can readily apply Theorem~\ref{thm:main_body} with $\TPop$ being the identity operator $\II_\HH$.

\begin{corollary}[Benefit of big step size for classification task.]\label{cor:numerical_bound_hilbert_norm}
    Under the settings of Theorem~\ref{thm:main_body} with the additional assumption that $\TPop = \II_{\HH_n}$, we have
    \begin{equation*}
        \LPop(\theta_b)- \LPop(\nu^\star) ~\leq~ \frac{34}{\kappa_\TEmp} (\LPop(\theta_s) - \LPop(\nu^\star)).
    \end{equation*}
\end{corollary}



\section{Experiments}

We evaluate the claims of \cref{sec:application_kernel_low_noise} on CKN-MNIST, a dataset consisting of the MNIST dataset embedded by a convolutional kernel network \citep{mairal:hal-01387399}. It allows for a realistic use-case, with classification accuracy close to $99\%$, by necessitating a reasonable number of samples $n = 1000$.
On CKN-MNIST, we achieve $98.5\%$ test accuracy with the Gaussian kernel with scale parameter $30$ and no regularization. Adding regularization only improves the test accuracy by $0.04\%$.

\begin{figure}[p]
    \centering
    \includegraphics[scale=.9]{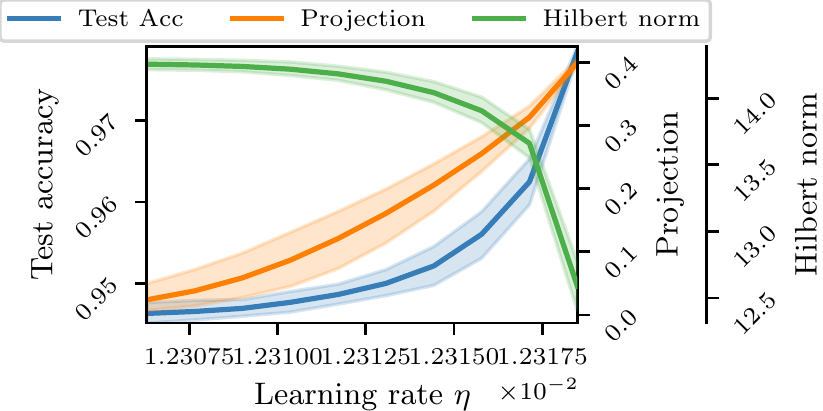}
    ~~
    \includegraphics[scale=.9]{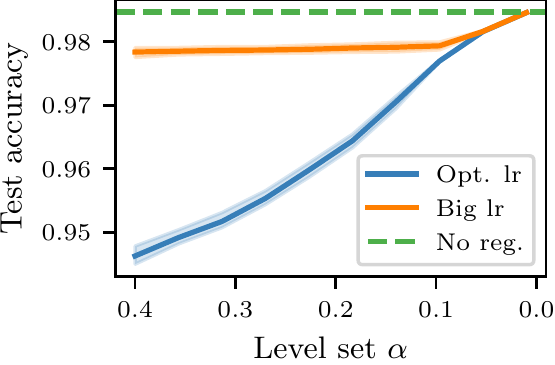}
    \caption{
        \emph{(Left)} Test accuracy \emph{(blue)} function of step size $\eta$ for CKN-MNIST. As the learning rate increases, the projection on the first component \emph{(orange)} increases, which makes the Hilbert norm \emph{(green)} decreases. This results in predictions closest in $\LL_\infty$ norm to the optimum.
        \\
        \emph{(Right)} Test accuracy function of level set $\alpha$ for CKN-MNIST. As we optimize more, the better performances of big step size \emph{(orange)} compare to optimal step size \emph{(blue)} vanish to reach the prediction of the optimum of $\LEmp$ \emph{(green, dashed)}.
        \\
        Shaded areas show standard deviation (train set and initialization) over 10 runs.
    }
    \label{fig:xps}
\end{figure}


\begin{figure}[p]
    \centering
    \subfigure[$s=30$, $\kappa_\TEmp=1.3 \cdot 10^4$][t]{
        \includegraphics[width=.48\textwidth]{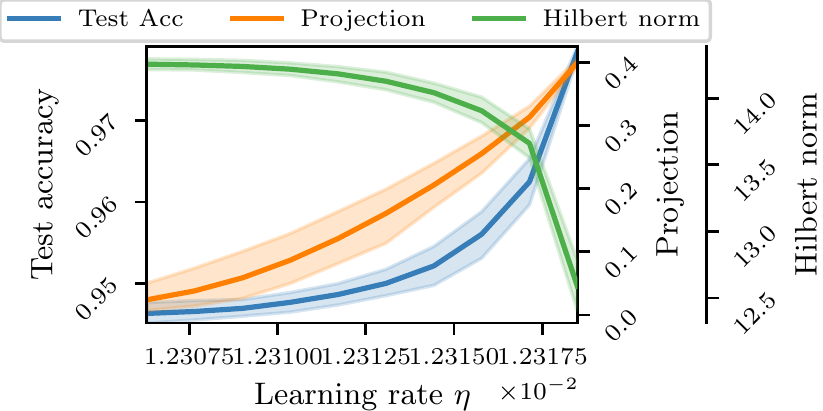}
    }
    \hfill
    \subfigure[$s=60$, $\kappa_\TEmp=2.8 \cdot 10^5$][t]{
        \includegraphics[width=.48\textwidth]{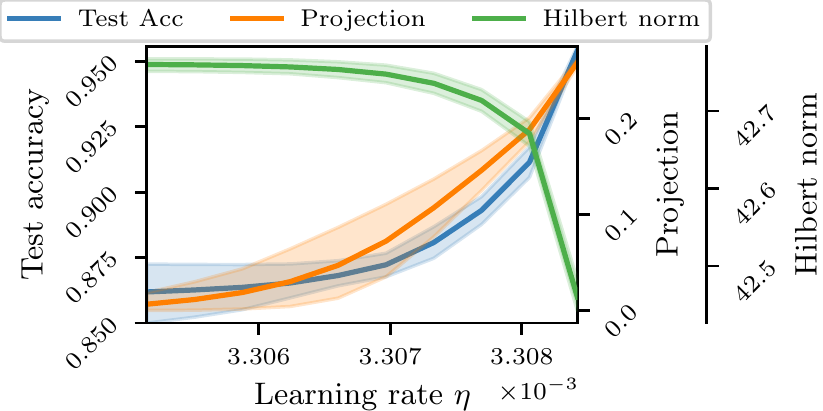}
    }
    \subfigure[$s=90$, $\kappa_\TEmp=1.4 \cdot 10^6$][t]{
        \includegraphics[width=.48\textwidth]{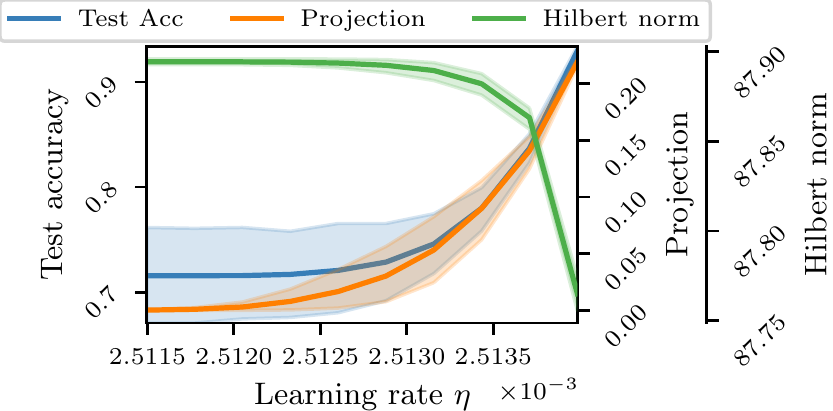}
    }
    \subfigure[$s=120$, $\kappa_\TEmp=4.0 \cdot 10^7$][t]{
        \includegraphics[width=.48\textwidth]{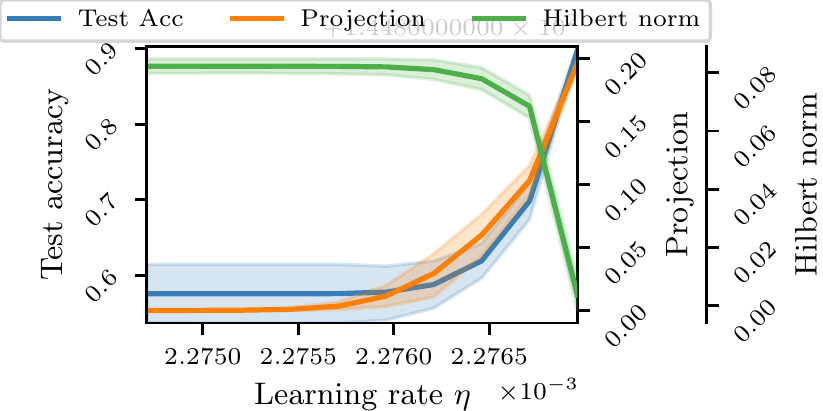}
    }

    \caption{Theorem~\ref{thm:main_body} predicts that the improvement in taking big rather than small step size increases with the condition number of $\TEmp$. We test this claim on our dataset: larger kernel scale $s$ makes the condition number of the kernel matrix $\kappa_\TEmp$ increases, which results in a larger margin between the test accuracy of $\theta_b$ and the one of $\theta_s$.}
    \label{fig:xps_multi_scale}
\end{figure}

\paragraph{Test accuracy function of the step size.}
We define some final level set $\alpha$. We plot three quantities function of the learning rate $\eta$. The \emph{projection on the first component} $\dotprod{\theta - \optEmp}{e_1}_\HH$ must be small for moderate learning rate and big for learning rate close to $2/\sigma_1$, as the attenuation satisfies $\absv{1 - \eta \sigma_1} \to 1$ when $\eta \to 2/\sigma_1$, see \cref{fig:attenuation_coefficient_function_lr}. This makes the \emph{Hilbert norm} $\LPop(\theta(\eta))$ decreases as $\eta$ increases, as predicted by Corollary~\ref{cor:numerical_bound_hilbert_norm}. This results in better test accuracy $1-B$ for big learning rate, following Lemma~\ref{lem:small_hilbert_norm_statistically_optimal}. This is summed up in \cref{fig:xps}, left. Note that the range for big step size is narrow. This is consistent with Assumption~\ref{asmpt:lr}, which requires a range equal to $2 \sigma_n / (\sigma_1 + \sigma_n) \approx 2\kappa_T^{-1}$. This is also in line with practices in deep learning where the best performance in generalization is often obtained by choosing the largest possible learning rate such that the model does not diverge.

\paragraph{Test accuracy function of optimization.}
Our main bound in Theorem~\ref{thm:main_body} relies on Assumption~\ref{asmpt:target}: in learning settings, it means that the optimization error $\alpha$ must be greater than a constant times the statistical error $\LPop(\optEmp)$ in order to observe improvements with big step sizes. This is shown in \cref{fig:xps}, right. Additional details on the plot is available in \cref{sec:appendix_experiments}.

\paragraph{Scale of the kernel.}
In \cref{fig:xps}, the scale of the Gaussian kernel is set to $s=30$, with which we obtained the best results on the test set. It is worth noting though that the scale of the kernel directly impacts the conditioning of the matrix. Notably, when $s \to 0$, the kernel matrix $K$ tends to the identity (hence $\kappa_\TEmp \to 1$) while when $s \to \infty$, $K$ tends to a rank-1 operator (hence $\kappa_\TEmp \to \infty$). A core result of Theorem~\ref{thm:main_body} is that we have \emph{bigger improvement for bigger $\kappa_\TEmp$}. We reproduce the experiment on the test accuracy with different scale in \cref{fig:xps_multi_scale}, and notice indeed bigger improvements for larger scale $s$, hence worse conditioning.

\section{Conclusion}

A large class of learning problems can be formulated as optimizing a function $\LEmp$ with gradient descent while we are interested in optimizing another function $\LPop$. Using simple quadratic forms to model this mismatch is a natural thing to do, while already providing lots of insight. We indeed show that the choice of large step sizes that may be suboptimal from an optimization point of view may provide better estimators than small/medium step sizes. In particular, we show that this phenomenon occurs in realistic classification tasks with low noise when learning with kernel methods.
In future work, we are planning to study other variants of gradient-based algorithms, which may be stochastic, or accelerated, and perhaps exploit the insight developed in our work to design new algorithms, which would focus on statistical efficiency, exploiting prior knowledge on the test loss, rather than on optimization of the training objective.

\acks{
    A.R. acknowleges support of the French government under management of Agence Nationale de la Recherche as part of the ``Investissements d'avenir'' program, reference ANR-19-P3IA-0001 (PRAIRIE 3IA Institute). A.R. acknowledges support of the European Research Council (grant REAL 947908).  J. Mairal was supported by the ERC grant number 714381 (SOLARIS project) and by ANR 3IA MIAI@Grenoble Alpes, (ANR19-P3IA-0003).
}

\bibliography{mybib}

\newpage
\input{sections/appendix.tex}
\end{document}

%% file: figures/attenuation_coeff.tex
\begin{figure}[h]
    \centering
    \begin{tikzpicture}
     
        \begin{axis}[
            xmin = 0, xmax = 2.1,
            ymin = 0, ymax = 1.1,
            ytick distance = 0.5,
            minor tick num = 1,
            width = .7\textwidth,
            height = 0.3\textwidth,
            xlabel = {$\eta$},
            ylabel = {$\absv{1- \eta \sigma_i}$},
            legend style={
                at={(0,0)},
                anchor=south west,
                nodes={scale=0.5, transform shape},
                },
            xtick={
                2,
                2/(1+.2),
                1
                },
            xticklabels={
                $2/\sigma_1$,
                $2/(\sigma_1+\sigma_n)$,
                $1/\sigma_1$
                },
            ticklabel style={font=\small}
    ]
         
        \addplot[
            domain = 0:2.1,
            samples = 200,
            smooth,
            thick,
            red
        ] {sqrt((1 - x)^2)};
    
        \addplot[
            domain = 0:2.1,
            samples = 200,
            smooth,
            thick,
            red!30!white
        ] {sqrt((1 - .9 * x)^2)};
    
        \addplot[
            domain = 0:2.1,
            samples = 200,
            smooth,
            thick,
            blue!30!white
        ] {sqrt((1 - .3 * x)^2)};
    
        \addplot[
            domain = 0:2.1,
            samples = 200,
            smooth,
            thick,
            blue
        ] {sqrt((1 - .2 * x)^2)};
    
        \addplot[
            dashed,
            black
        ]
        coordinates {(2, 0) (2, 1.1)};
        \addplot[
            dashed,
            black
        ]
        coordinates {(2/(1+.2), 0) (2/(1+.2), 1.1)};

        \legend{
            $\sigma_1 = 1.$, 
            $\sigma_2 = 0.9$, 
            $\sigma_{n-1} = 0.3$, 
            $\sigma_n = 0.2$, 
        }
    \end{axis}
    \end{tikzpicture}
    \vspace{-1ex}
    \caption{Attenuation coefficient $\absv{1 - \eta \sigma_i}$ function of the step size $\eta$, for 4 eigenvalues. For $1 \leq i \leq n$, the attenuation is the quantity by which decays the projection of $\theta - \optEmp$ on $e_i$ at each step. The closer to $0$, the faster will the direction $e_i$ of $\TEmp$ be learnt. In our analysis, the learning rate must satisfy $ \eta_s < 2/(\sigma_1 + \sigma_n) < \eta_b < 2/\sigma_1$.}
    \label{fig:attenuation_coefficient_function_lr}
\end{figure}

%% file: sections/appendix.tex
\appendix

\section*{Overview of the appendix}
\begin{itemize}
    \item In \cref{sec:appendix_proof_main_thm}, we prove Theorem~\ref{thm:main_body}.
    \item In \cref{sec:appendix_low_noise_loucas_results}, we give additional technical information on the low-noise classification task.
    \item We compare the bound in Theorem~\ref{thm:main_body} to existing results in the case of regression task with kernels, in \cref{sec:appendix_comparison_spectral_filters}. We highlight the difference between gradient descent with big step size and the estimators which write as a spectral filter.
    \item \Cref{sec:appendix_gradient_descent_in_practice} makes a simple remark, highlighting the difference between gradient descent on the train loss and gradient descent on the Hilbert norm in kernel regression.
    \item Finally, \cref{sec:appendix_experiments} gives additional information on the experiments.
\end{itemize}

\section{Proof of main result}
\label{sec:appendix_proof_main_thm}

\input{sections/new_proof.tex}






\section{Low-noise classification tasks}
\label{sec:appendix_low_noise_loucas_results}

Taking big step size is particularly critical in classification tasks. In this section, we build on the result of \cite{pmlr-v75-pillaud-vivien18a} to relate classification performances with Hilbert norm. Recall notably the notations of \cref{sec:application_kernel_low_noise}.

\paragraph{Assumptions.}
The following assumption comes from (A1) in \cite{pmlr-v75-pillaud-vivien18a}. It is well characterized in usual image classification settings.

\begin{assumption}[Strong margin condition.]\label{asmpt:strong_margin_condition_a1}
    We have $\regFunc(x) \geq \delta$ for some $\delta \in \p{0, 1}$.
\end{assumption}

The second assumption characterizes the statistical optimality of $\optPop$. It does not assume the regression function to belong to $\HH$, but ensures some proximity in $\LL_\infty$ norm. It is close to (A4) in \cite{pmlr-v75-pillaud-vivien18a}.

\begin{assumption}[Statistical optimality of the population loss' optimum.]\label{asmpt:statistical_optimality_empirical_loss_a4}
    We have that
    \begin{equation}
        \sign \p{\regFunc(x)} \optPop(x) \geq \delta/2.
    \end{equation}
\end{assumption}
This assumption is satisfied as soon as the regression function can be approximated by a function of the RKHS with precision $\delta/2$ in $\LL_\infty$ norm. For instance, a sufficient condition is the regression function  $\regFunc(x)$ to belong to $\HH$. Then $\regFunc = \optPop$ and the assumption is satisfied. Note that this always imply that $\optEmp$ reaches $0$ test error for sufficienly many samples, which is the key hindsight of \cite{pmlr-v75-pillaud-vivien18a}. Indeed, for a proper choice of regularization $\lambda$ one has that
\begin{equation*}
    \norm{\optEmp - \optPop}_\HH \lesssim n^{-\frac{b r}{1 + b(2 r + 1)}},
\end{equation*}
where $(r, b)$ are the parameters of the source and capacity condition, both of which characterizes the difficulty of the learning task, see \cite{blanchard}. This implies that for sufficiently many training samples $n$, $\optEmp$ will be close to $\optPop$ in Hilbert norm, which implies proximity in $\LL_\infty$ (pointwise) norm.

\paragraph{Hilbert norm proximity implies statistical optimality}

The following lemma is very close to Lemma 1 in \cite{pmlr-v75-pillaud-vivien18a}, and is a direct consquence of our assumption. We first introduce
\begin{equation}\label{eq:def_Omega_space_pos_neg}
    \Omega_+ = \cb{x; \regFunc(x) \geq \delta}, ~~ \Omega_- = \cb{x; \regFunc(x) \leq -\delta}.
\end{equation}
Next lemma basically relies on the decomposition
\begin{equation*}
    \norm{\theta - \regFunc}_{\LL_\infty} \leq \norm{\theta - \optPop}_{\LL_\infty} + \norm{\optPop - \regFunc}_{\LL_\infty}.
\end{equation*}

\begin{lemma}[Small Hilbert norm implies statistical optimality]\label{lem:small_hilbert_norm_statistically_optimal}
    Consider an estimator $\theta$ which satisfies
    \begin{equation*}
        \norm{\theta - \optPop}_\HH \leq \frac{\delta}{2 C_K}.
    \end{equation*}
    Then, this estimator is statistically optimal, in the sense that it has 0 excess error:
    \begin{equation*}
        B(\theta) - \inf_{\theta \in \HH} B(\theta) = 0.
    \end{equation*}
\end{lemma}

\begin{proof}
    First of all, we leverage the fact that the Hilbert norm upper bounds the $L_\infty$ norm, with
    \begin{equation*}
        \norm{\theta - \optPop}_{L_\infty} \leq C_K \norm{\theta - \optPop}_\HH \leq \frac{\delta}{2}.
    \end{equation*}
    Then, on $\Omega_+$, whose definition is given in \cref{eq:def_Omega_space_pos_neg}, we have that
    \begin{equation*}
        \forall x \in \Omega_+, \theta_{x} > \regFunc(x) - \norm{\theta - \optPop}_{L_\infty} - \norm{\optPop - \regFunc}_{L_\infty} \geq \delta - \frac{\delta}{2} - \frac{\delta}{2} = 0,
    \end{equation*}
    so $\theta$ will have accurate prediction for all positive labels. The same goes for negative labels. Thus, $\theta$ has 0 test error.
\end{proof}
Thus, we see that the \emph{Hilbert norm is a good proxy for minimizing the test error $B$}.

\section{Regression tasks and comparison with spectral filters}
\label{sec:appendix_comparison_spectral_filters}

If the downstream task is regression, then we can still apply our result by introducing the \emph{population covariance operator},
\begin{equation}\label{eq:def_population_covariance}
    \TPop = \int \phi(x) \otimes \phi(x) \dd \rho_x(x).
\end{equation}
Then, Theorem~\ref{thm:main_body} holds by considering (recall the definition of the population loss $\calP$ in \cref{eq:def_pop_loss})
\begin{equation*}
    \LPop(\theta) = \frac{1}{2} \norm{\theta - \optPop}_{\TPop}^2 = \mathcal{P}(\theta) - \inf_{\nu \in \HH} \mathcal{P}(\nu),
\end{equation*}
which is nothing but the \emph{excess risk} of the estimator. Then, under the assumptions of Theorem~\ref{thm:main_body} we have that
\begin{equation}\label{eq:bound_regression}
    \LPop(\theta_b) ~\leq~ 34 \frac{\kappa_{\TPopProj}}{\kappa_\TEmp} \, \LPop(\theta_s).
\end{equation}

Gradient descent for kernel ridge regression has been widely studied in the past, to say the least. \Cref{eq:bound_regression} appears to be in contradiction with most of them. In this section, we emphasize the limit of our assumptions to point out that there is no conflict with existing theory.

\paragraph{Early stage of training.}
The bound in \cref{eq:bound_regression} ensures better generalization when taking big step size, if the r.h.s is bigger than 1. However, the pioneering work of \cite{yao2007early} established that the learning rate had no influence in the generalization capabilities of the estimator. A key difference though is that the results of Theorem~\ref{thm:main_body} only holds in the early stage of training, when the optimization error $\alpha$ is big w.r.t to the statistical error $\LPop(\optEmp)$ : otherwise, Assumption~\ref{asmpt:target} is not satisfied. In contrast, results of the like of \cite{yao2007early} holds for sufficiently many samples $n$, and require a number of steps $t$ bounded by below by a power of $n$ -- they require an upper-bound on $\alpha$, while we require a lower-bound in Assumption~\ref{asmpt:target}.

\paragraph{Is \cref{eq:bound_regression} informative?}
As mentioned above, \cref{eq:bound_regression} ensures better generalization of big step size only if the r.h.s is bigger than $1$. However, in the particular scenario of kernel regression, the empirical covariance $\TEmp$ is the \emph{discretization} of the population covariance $\TPop$. Thus, numerous results bound the discrepancy between the two, notably when the capacity condition holds, see \textit{e.g.} Proposition 5.3 to 5.5 in \cite{blanchard}. In these settings, we can expect the ratio $\kappa_\TEmp/\kappa_\TPop$ to go to $1$ for large number of samples $n$. Thus, we cannot conclude in better excess risk of $\theta_b$ compared to $\theta_s$.

\paragraph{Comparison with spectral filters.}
Spectral filters are an elegant way to describe a wide family of regularization for kernel regression \cite{pub.1045202542,bauer2007regularization}. In a nutshell, it relies on studying the class of estimator characterized by a filter function $g_\lambda$, where $\lambda$ is a regularization parameter, equal to $1/t$ in the case of early stopping in GD. GD with moderate step sizes is a spectral filter; but GD with large step size is not. We now explain this difference, which helps to build an intuition on our result.

Consider the estimator $\theta = g_\lambda(\TEmp) S^* y$, where $S$ is the so-called sampling operator defined in \cref{sec:appendix_operators}. The unregularized solution is obtained with $\lambda=0$, for which we must have $g_{\lambda=0}(\sigma) = \sigma^{-1}$. We denote it with $\optEmp = \TEmp^{-1}S^*y$, and we can see how does $\theta$ approaches the unregularized optimum. We have
\begin{align*}
    \dotprod{\theta - \optEmp}{e_i}
     & = \dotprod{g_\lambda(\TEmp)S^* y - \TEmp^{-1}S^* y}{e_i}             \\
     & = \dotprod{\p{g_\lambda(\TEmp)\TEmp^{-1} - \II}\TEmp^{-1}S^* y}{e_i} \\
     & = \dotprod{\p{g_\lambda(\TEmp)\TEmp - \II}\optEmp}{e_i}              \\
     & = (g_\lambda(\sigma_i) \sigma_i - 1) \dotprod{\optEmp}{e_i}.
\end{align*}
If we start from $\theta_0 \neq 0$, this relation turns to $\absv{\dotprod{\theta - \optEmp}{e_i} } = \absv{1 - g_\lambda(\sigma_i) \sigma_i} \absv{\dotprod{\theta_0 - \optEmp}{e_i}}$ in the case of GD.
We denote the residual with $r_\lambda(\sigma) = \absv{1 - \sigma g_\lambda(\sigma)}$. We then compare $r_\lambda$ for various spectral filters in \cref{fig:comparison_spectral_filters}. Note that for gradient descent, we have $r_{1/t}(\sigma) = \absv{1 - \eta \sigma}^t$ and we recover the expression we obtained from \cref{eq:gradient_descent_update}. The key hindsight is that spectral filters will learn, \textit{i.e} \emph{optimize}, the \emph{biggest eigendirection} first. For instance, truncated regression uses as estimator the first eigencomponents of the unregularized estimator, leaving the smaller eigencomponents untouched. This is at odds with what we aim at with big learning rate. There is no contradictions though, as we want in the end to minimize the excess risk $\LPop$ -- a quadratic with operator $\TPop$ -- and we assumed the empirical covariance $\TEmp$ to be a discretization of $\TPop$. Thus, in this settings $\LEmp$ is a good proxy for $\LPop$ and minimizing the biggest eigendirection first will make the excess risk $\LPop$ decrease faster. This corresponds to having level sets well aligned in \cref{fig:estimators_path}.

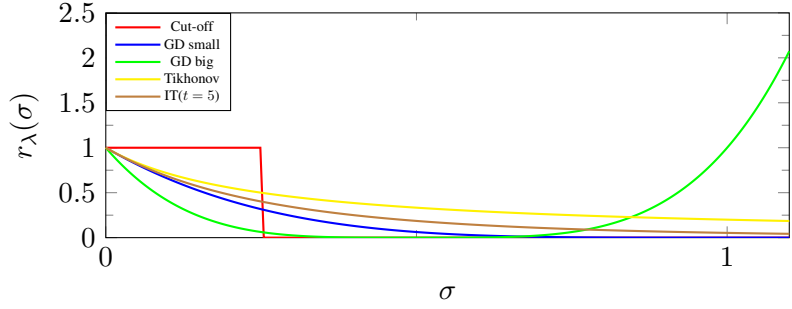
\begin{figure}[h]
    \centering
    \begin{tikzpicture}

        \begin{axis}[
                xmin = 0, xmax = 1.1,
                ymin = 0, ymax = 2.5,
                ytick distance = 0.5,
                minor tick num = 1,
                width = .7\textwidth,
                height = 0.3\textwidth,
                xlabel = {$\sigma$},
                ylabel = {$r_\lambda(\sigma)$},
                legend style={
                        at={(0,1)},
                        anchor=north west,
                        nodes={scale=0.5, transform shape},
                    },
                xtick={0, 1},
                xticklabels={$0$, $1$},
            ]

            \addplot[
                domain = 0:1.1,
                samples = 200,
                thick,
                red
            ] {(x<1/4)};

            \addplot[
                domain = 0:1.1,
                samples = 200,
                smooth,
                thick,
                blue
            ] {(1-x)^4};

            \addplot[
                domain = 0:1.1,
                samples = 200,
                smooth,
                thick,
                green
            ] {(1-2*x)^4};

            \addplot[
                domain = 0:1.1,
                samples = 200,
                smooth,
                thick,
                yellow
            ] {1- x/(x+1/4)};

            \addplot[
                domain = 0:1.1,
                samples = 200,
                smooth,
                thick,
                brown
            ] {(1/(1+x*(4/5)))^5};

            \legend{
                Cut-off,
                GD small,
                GD big,
                Tikhonov,
                IT($t=5$)
            }

        \end{axis}
    \end{tikzpicture}
    \caption{Residual of various spectral filters, with regularization $\lambda=1/4$ or $t=4$ for GD. The best filter is spectral cut-off \emph{(red)}. The resulting estimator is purely directed along the smallest eigenvectors of $\TEmp$. Gradient descent with small step sizes \emph{(blue)} and (iterated) Tikhonov \emph{(yellow, brown)} mimick this filter. On the other hand, gradient descent with big step sizes \emph{(green)} is not an admissible filter in the sense of \cite{pub.1045202542}, as it attenuates less  the biggest component.}
    \label{fig:comparison_spectral_filters}
\end{figure}

\paragraph{Theorem~\ref{thm:main_body} in practice.}
The limits of this subsection -- low optimization regime, low value for $\kappa_\TEmp/\kappa_\TPop$ and difference with spectral filters -- can be mitigated for multiple reasons. First of all, as we discussed earlier kernel regression can simply be a mean in order to solve a \emph{classification task}, in which case the statistical results of spectral filters are no longer relevant. Secondly, there can be big discrepancies between the train and test set in practice. Indeed, the risk $\LPop$ with which estimators are compared is often a separate test set, with fixed condition number $\kappa_\TPop$. Additionally, data augmentation can be used on the train set, which then introduces spurrious directions in the empirical covariance matrix $\TEmp$. Thus, even though spectral filters are optimal in theoretical settings, taking big step size can prove useful in practical scenari, which are covered by our settings with quadratic forms of $\HH$.

\section{Gradient descent updates in practice}
\label{sec:appendix_gradient_descent_in_practice}

\input{sections/useful_operators.tex}

\subsection{Gradient descent on the Hilbert norm is possible}
\paragraph{Different spectrum between $\HH$ and $\RR^n$.}
We denote the training loss with $\LEmp$ and the Hilbert norm with $\LHilbert{}$. Given the relation of \cref{eq:relation_alpha_theta}, we have that
\begin{equation}
    \begin{aligned}
        \LEmp\p{\theta}   & = \frac{1}{2} \norm{\theta - \optEmp}_{\TEmp}^2 = \frac{1}{2n} \norm{K(\alpha - \alphaOpt)}_{\RR^n}^2, \\
        \LHilbert{\theta} & = \frac{1}{2} \norm{\theta - \optEmp}_{\HH}^2 = \frac{1}{2} \norm{\alpha- \alphaOpt}_{K}^2,
    \end{aligned}
\end{equation}
where we overloaded $\LEmp$ to be a function of $\HH_n$ \textit{and} $\RR^n$. Specifically, we used $\LEmp\p{\alpha} = \LEmp \circ \sqrt{n} \SEmpD \p{\alpha}$.
Recall that $K/n$ and $\TEmp$ share the same spectrum. Thus, $\LEmp$ is a quadratic whose spectrum is $\p{\sigma_1, \dots, \sigma_n}$ w.r.t the variable $\theta$, but spectrum $\p{n\sigma_1^2, \dots, n\sigma_n^2}$ w.r.t the variable $\alpha$. Likewise, $\LHilbert{}$ is a quadratic with spectrum $\p{1, \dots, 1}$ w.r.t the variable $\theta$, but a spectrum $\p{n\sigma_1, \dots, n\sigma_n}$ w.r.t the variable $\alpha$.

\paragraph{Do we care about this difference?}
The global picture behind what follows is that $\alpha$ is isomorphic to $\TEmp^{-1/2} \HH$. If we expressed the estimator $\theta$ as a combination of eigenbasis vector, that is $\theta = \sum_i \beta_i e_i$, then $\beta$ is isormorphic to $\HH$ and the distinction does not hold. The fact that the estimator writes as a combination of $\phi(x_i)$ with $\alpha$ adds another level of geometric distortion.

\paragraph{Gradient descent in $\HH$.}
In the Hilbert space $\HH$, the updates are easy:
\begin{equation*}
    \begin{aligned}
        \theta_{t+1} & = \theta_{t} - \eta \TEmp (\theta_{t} - \optEmp) \iff \theta_{t} - \optEmp & = \p{\II - \eta \TEmp}^t(\theta_{0} - \optEmp), &  & \text{(GD on $\LEmp$)}        \\
        \theta_{t+1} & = \theta_{t} - \eta (\theta_{t} - \optEmp) \iff \theta_{t} - \optEmp       & = \p{1-\eta}^t(\theta_{0} - \optEmp),           &  & \text{(GD on $\LHilbert{}$)}.
    \end{aligned}
\end{equation*}
The big learning rate range is then $\eta_s < 2/(\sigma_1 + \sigma_n) < \eta_b < (2/\sigma_1)$.

\paragraph{Gradient descent in $\RR^n$.}
In practice, we do not have access to $\alphaOpt$, or only though it's evaluation with $K$. Yet, we are still able to minimize these quadratic form through the gradient. \textit{E.g.} when $\alphaOpt$ is defined through $K \alphaOpt = y$ in the unregularized settings, or $(K + n\lambda) \alphaOpt = y$ in the Tikhonov-regularized case. The gradient descent updates on the train loss read:
\begin{equation}\label{eq:in_practice_gradient_updates_train_loss}
    \begin{aligned}
        \alpha_{t+1}
         & = \alpha_t - \eta \frac{K^2}{n} (\alpha_t - \alphaOpt)
        = \alpha_t - \eta \frac{K}{n} (K \alpha_t - y)                                                               \\
         & \iff \alpha_t - \alphaOpt = \p{\II - \eta\frac{K^2}{n}}^t (\alpha_0 - \alphaOpt)
         &                                                                                  & \text{(GD on $\LEmp$)}
    \end{aligned}
\end{equation}
Here, the range of learning rate is $\eta_s < 2 / \csb{n (\sigma_1^2 + \sigma_n^2)} < \eta_b < 2/\csb{n\sigma_1^2}$.
Interestingly, we can still do gradient descent on the Hilbert norm in closed form!
\begin{equation}\label{eq:in_practice_gradient_updates_hilbert}
    \begin{aligned}
        \alpha_{t+1}
         & = \alpha_t - \eta K (\alpha_t - \alphaOpt)
        = \alpha_t - \eta (K \alpha_t - y)                                                                      \\
         & \iff \alpha_t - \alphaOpt = \p{\II - \eta K}^t (\alpha_0 - \alphaOpt)
         &                                                                       & \text{(GD on $\LHilbert{}$)}
    \end{aligned}
\end{equation}
Here, the optimal learning rate is $1/[n\sigma_1]$. Interestingly, choosing a big learning rate in the range $1/\csb{n (\sigma_1 + \sigma_n)} < \eta_b < 2/\csb{n\sigma_1}$ results in an estimator which is closed in \emph{euclidean} norm ( $\RR^n$) to $\alphaOpt$. Note that even though we can evaluate the gradient of $\LHilbert{}$, we \emph{cannot} evaluate its value. Indeed, the objective function would read
\begin{equation*}
    \frac{1}{2} \norm{\alpha - \alphaOpt}^2 = \frac{1}{2} \norm{\alpha - K^{-1}y}^2
\end{equation*}
which is not accessible without inverting the (regularized) kernel matrix.

\section{Additional details on the experiment}
\label{sec:appendix_experiments}

\paragraph{Setting the learning rate.}
We give additional details on the plot ``test accuracy function of train loss $\alpha$'' in \cref{fig:xps}.
The plot is averaged over 10 initialization for $\theta_0$. We used $\eta_s = 1/\sigma_1$ and $\eta_b = \tau \cdot 2/\sigma_1$, with $\tau = 1 - 10^{-5}$. We elaborate on these choices:
\begin{itemize}
    \item The optimal learning rate for upper bounding for $\sigma_1$-smooth, $\sigma_n$-strongly convex function is $\eta_\text{opt} = 2 / (\sigma_1 + \sigma_n)$, as explained in the discussion of Assumption~\ref{asmpt:lr}. However, this requires a massive amount of steps to converge. This is due to the terms depending on the initialization in the lower bound for $t_b, t_s$ in \cref{eq:lower_bound_t_big,eq:lower_bound_t_small}. Thus, we set $\eta_s = 1/\sigma_1$, which is the optimal rate for $\sigma_1$-smooth function, and we do observe fast convergence with this choice.
    \item Instead of choosing $\eta_b \in \csb{2/(\sigma_1 + \sigma_n), 2/\sigma_1}$, we use $\eta_b = \tau \cdot 2/\sigma_1$, with $\tau$ chosen with the experiment on the test accuracy (\cref{fig:xps}, left). Indeed, setting $\tau = 1$ can result in situation where there can't be convergence; and choosing $\eta_b > \eta_\text{opt}$, as we describe in the theory, results in very slow convergence.
\end{itemize}
This discrepancy between theory and practice is due to our proof which is very conservative in the error bound. A more refined analysis would rely on $\theta_b - \optEmp$ (resp. $\theta_s - \optEmp$) belonging to the span of the $k$-th first (resp. last) eigenvectors. Besides, in practical settings the learning rate is an hyperparameter to tune, which is exactly the approach we used to produce \cref{fig:xps}.

%% file: sections/new_proof.tex


\subsection{Definition and assumptions}

Recall from the main text Assumptions \ref{asmpt:notation_operator},
\ref{asmpt:lr},
\ref{asmpt:init_iota},
\ref{asmpt:target}.

We can rewrite precisely the gradient descent update in \cref{eq:gradient_descent_update} with the following definition.
\begin{definition}[Notations for the estimator]\label{def:notation_estimator}
    For some step size $\eta$, a number of steps $t$ and some $\theta_0 \in \HH_n$, we denote $\theta^{(\eta, t)}$ the estimator obtained through gradient descent from $\theta_0$. We denote $\p{\mu_{i}^{(\eta,t)}}_{1 \leq i \leq n}$ the decomposition of $\theta^{(\eta, t)} - \optEmp$ on $e_i$, \emph{i.e.}
    \begin{equation}
        \theta^{(\eta, t)} - \optEmp = \sum_{i = 1}^n \mu_i^{(\eta, t)} e_i.
    \end{equation}
    Denoting the initialization with $(\iota_i)_{1 \leq i \leq n}$,
    \begin{equation}
        \theta_0 - \optEmp = \sum_{i=1}^n \iota_i e_i,
    \end{equation}
    we have that
    \begin{equation}\label{eq:def_attenuation}
        \forall i \in \cb{1, \dots, n}, ~~ \mu_{i}^{(\eta, t)} = \iota_i A_i^{(\eta, t)} = \iota_i (1 - \eta \sigma_i)^t,
    \end{equation}
    where $A_i$ is the attenuation of the $i$-th eigencomponent at each step.

    To lighten the notations, we denote $\theta_s = \theta^{(\eta_s, t_s)}$ the estimator obtained with $t_s$ \emph{small} step size $\eta_s$, and $\theta_b = \theta^{(\eta_b, t_b)}$ the estimator obtained with $t_b$ \emph{big} step size $\eta_b$. Likewise, we use $\mu_i$ when it is clear from the context which of $\theta_s$ or $\theta_b$ we study.
\end{definition}

With these notations and Assumption~\ref{asmpt:lr} and \ref{asmpt:init_iota}, note that
\begin{equation*}
    \forall (\eta, t), ~~ \mu_1^{(\eta_b, t)} \neq 0, ~~ \mu_n^{(\eta_s, t)} \neq 0.
\end{equation*}

Also, we can now introduce the \emph{second biggest attenuation coefficients} in the following definition.

\begin{definition}[Second biggest attenuation coefficient.]
    We introduce the second biggest attenuation coefficients $\bar{A}_b, \bar{A}_s$ with
    \begin{equation}
        \bar{A}_b \eqdef \max\cb{A_2^{(\eta_b)}, A_n^{\eta_b}}, ~~ \bar{A}_s \eqdef A_{n-1}^{(\eta_s)}.
    \end{equation}
\end{definition}
Referring to \cref{fig:attenuation_coefficient_function_lr}, this implies
\begin{itemize}
    \item For $\eta_b$, we have that
          \begin{equation}
              A_1 > \max\cb{A_2, A_n} \eqdef \bar{A}_b \geq A_i, ~~ \forall i > 1.
          \end{equation}
          Thus, by tuning the number of steps $t$, we can make the ratio $(A_i / A_1)^t$ arbitrarily small for any $i > 1$.
    \item For $\eta_s$ we have that
          \begin{equation}
              A_n > A_{n-1} \eqdef \bar{A}_s \geq A_i, ~~ \forall i < n.
          \end{equation}
          Again, for a sufficiently large number of steps $t$, we can make $(A_i / A_n)^t$ arbitrarily small for any $i < n$.
\end{itemize}

\begin{definition}[Upper bound on $\alpha$.]\label{def:alpha_1}
    Given some small and big learning rate $\eta_s, \eta_b$, we introduce $\alpha_1$, a technical quantity depending on the spectrum of $\TEmp$ and $\TPop$ and the initialization:
    \begin{equation}
        \alpha_1 = \frac{1}{2} \sigma_n \iota_n^2 \exp \p{-
            \frac{
                \log \csb {\norm{\theta_0 - \optEmp}_\HH^2 \max \cb{16n\kappa_{\TPop}, 4\kappa_\TEmp} \cdot \max \cb{
                        \frac{1}{\iota_1^2}, \frac{1}{\iota_n^2}
                    }
                    + \frac{1}{1 - \eta_s \sigma_n}
                    + \frac{1}{\eta_b \sigma_1 - 1}
                }}{
                \min \cb{
                    \log \frac{1 - \eta_s\sigma_{n}}{1 - \eta_s\sigma_{n-1}},
                    \log \frac{A_1}{\bar{A}_b}
                }}}
    \end{equation}
\end{definition}

In particular, note that we have
\begin{align*}
    \alpha_1 & \leq \frac{1}{2}\sigma_1 \iota_1^2 \exp \p{- \frac{\log \csb{\frac{\norm{\theta_0 - \optEmp}_\HH^2}{\iota_1^2} 4n\kappa_{\TPop} + \frac{1}{\eta_b \sigma_1 - 1}}}{\log \frac{A_1}{\bar{A}_b}}},                                                          \\
    \alpha_1 & \leq \frac{1}{2} \sigma_n \iota_n^2 \exp \p{- \frac{\log \csb {\frac{\norm{\theta_0 - \optEmp}_\HH^2}{\iota_n^2} \max \cb{16n\kappa_{\TPop}, 4\kappa_\TEmp}+ \frac{1}{1 - \eta_s \sigma_n}}}{\log \frac{1 - \eta_s\sigma_{n}}{1 - \eta_s\sigma_{n-1}}}},
\end{align*}
which will prove useful for the derivation of Lemmas~\ref{lem:technical_bound_big_lr} and \ref{lem:technical_bound_small_lr}.

\subsection{An upper bound for the estimator with big learning rate}
In this subsection, we denote $\mu_i^{\eta_b, t_b}$ with $\mu_i$.

\begin{lemma}[Estimator with big step size.]\label{lem:technical_bound_big_lr}
    Set $\alpha > 0$ s.t.
    \begin{equation}
        \alpha < \alpha_1,
    \end{equation}
    where $\alpha_1$ is defined in Def.~\ref{def:alpha_1}. Define the quantity $\epsilon_b$ with
    \begin{equation}
        \epsilon_b^2 = \frac{\sum_{i > 1} \mu_i^2}{\mu_1^2}.
    \end{equation}
    Then, running gradient descent on $\LEmp$ with step size $\eta_b$ until the $(\alpha/2, \alpha)$ level sets are reached, \emph{i.e.}
    \begin{equation}
        \frac{1}{2} \alpha \leq \LEmp(\theta_b) \leq \alpha,
    \end{equation}
    ensures that
    \begin{equation}
        \epsilon_b^2 \leq \frac{1}{4 n \kappa_\TPop}, ~~ \andtext ~~ \frac{2}{5} \alpha \leq \frac{1}{2} \sigma_1 \mu_1^2 \leq \alpha.
    \end{equation}
    The resulting estimator is obtained with $t_b$ steps, with
    \begin{equation}\label{eq:lower_bound_t_big}
        t_b \geq \frac{1}{2} \frac{\log \frac{1}{2}\frac{\sigma_1 \iota_1^2}{\alpha}}{\log \frac{1}{\eta_b \sigma_1 - 1}} \geq \frac{1}{2}\frac{\log \frac{\norm{\theta_0 - \optEmp}_\HH^2}{\iota_1^2} 4n\kappa_{\TPop}}{\log \frac{A_1}{\bar{A}_b}}.
    \end{equation}
\end{lemma}

\begin{proof}
    ~\paragraph{Bound on $\epsilon_b$.}
    By using the definition of $\epsilon_b$ and the expression of the $(\mu_i)_{1 \leq i \leq n}$ given in Def.~\ref{def:notation_estimator}, we have
    \begin{equation}
        \epsilon_b^2 = \frac{\sum_{i > 1} \mu_i^2}{\mu_1^2} = \frac{\sum_{i > 1} \iota_i^2 A_i^{2t}}{\iota_1^2 A_1^{2t}} \leq \frac{\norm{\theta_0 - \optEmp}_\HH^2}{\iota_1^2} \p{\frac{\bar{A}_b}{A_1}}^{2t}.
    \end{equation}
    Thanks to the proper choice of $\eta_b$ given in Asmpt.~\ref{asmpt:lr}, we have that ${\bar{A}_b}/{A_1} < 1$, so that
    \begin{equation}\label{eq:proof_big_lr_delta}
        \forall \delta > 0, ~~ t \geq t_1 \eqdef \frac{1}{2}\frac{\log \frac{\norm{\theta_0 - \optEmp}_\HH^2}{\delta^2 \iota_1^2}}{\log \frac{A_1}{\bar{A}_b}} \implies \epsilon_b^2 \leq \delta^2.
    \end{equation}

    \paragraph{Bound on $\mu_1$.}
    Now, recall that the optimization error reads
    \begin{equation*}
        \LEmp(\theta_b) = \frac{1}{2} \sum_{i=1}^n \sigma_i \mu_i^2 = \frac{1}{2} \sigma_1 \mu_1^2 \p{1 + \frac{\sum_{i > 1} \sigma_i \mu_i^2}{\sigma_1 \mu_1^2}}
    \end{equation*}
    as we assumed that $\iota_1 \neq 0$ in Asmpt.~\ref{asmpt:init_iota}. Thus, we can bound the loss in two ways. First, by definition
    \begin{equation}\label{eq:proof_big_lr_1}
        \frac{1}{2} \sigma_1 \mu_1^2 \leq \LEmp(\theta_b) \leq \frac{1}{2} \sigma_1 \mu_1^2 (1 + \epsilon_b^2)
    \end{equation}
    and second, we assumed the estimator to belong to the $(\alpha/2, \alpha)$ level set of $\LEmp$, \emph{i.e.}
    \begin{equation}\label{eq:proof_big_lr_2}
        \frac{\alpha}{2} \leq \LEmp(\theta_b) \leq \alpha.
    \end{equation}
    Combining \cref{eq:proof_big_lr_1} with \cref{eq:proof_big_lr_2}, we have that
    \begin{equation}
        \frac{\alpha}{2 (1 + \epsilon_b^2)} \leq \frac{1}{2} \sigma_1 \mu_1^2 \leq \alpha.
    \end{equation}

    \paragraph{Feasability.}
    Finally, we consider if $\epsilon_b^2 \leq \frac{1}{4 n \kappa_\TPop}$ and $\theta_b$ being in the $(\alpha/2, \alpha)$ level sets can occur at the same time.

    First of all, we set the value of $\delta^2$ to $\frac{1}{4 n \kappa_\TPop}$ in \cref{eq:proof_big_lr_1}. We get
    \begin{equation}
        t_1 \eqdef \frac{1}{2}\frac{\log \frac{\norm{\theta_0 - \optEmp}_\HH^2}{\iota_1^2} 4n\kappa_{\TPop}}{\log \frac{A_1}{\bar{A}_b}}.
    \end{equation}

    Then, we derive necessary conditions for having \cref{eq:proof_big_lr_2}. Those conditions are derived through the bounds established in \cref{eq:proof_big_lr_1}. For the lower bound, assuming $t \geq t_1$, so that, in particular, we have $\epsilon_b^2 \leq 1/4$,
    \begin{align*}
        \frac{\alpha}{2} \leq \LEmp(\theta_b)
         & \implies \frac{\alpha}{2} \leq \frac{1}{2} \sigma_1 \mu_1^2 (1 + \epsilon_b^2)                                                       \\
         & \implies \frac{4\alpha}{5\sigma_1} \leq \mu_1^2 = \iota_1^2 (\eta_b\sigma_1 - 1)^{2t}                                                \\
         & \implies t \leq t_3 \eqdef \frac{1}{2} \frac{\log \frac{5}{4}\frac{\sigma_1 \iota_1^2}{\alpha}}{\log \frac{1}{\eta_b \sigma_1 - 1}}.
    \end{align*}
    Likewise, for the upper bound we have,
    \begin{align*}
        \LEmp(\theta_b)  \leq \alpha
         & \implies \frac{1}{2} \sigma_1 \mu_1^2 \leq \alpha                                                                                    \\
         & \implies t \geq t_2 \eqdef \frac{1}{2} \frac{\log \frac{1}{2}\frac{\sigma_1 \iota_1^2}{\alpha}}{\log \frac{1}{\eta_b \sigma_1 - 1}}.
    \end{align*}

    Summing up, we have
    \begin{itemize}
        \item $t \geq t_1$ implies that the bound on $\epsilon_b$ in \cref{eq:proof_big_lr_delta} holds.
        \item Ensuring that the level set condition in \cref{eq:proof_big_lr_2} holds are met implies that $t \in (t_2, t_3)$.
    \end{itemize}
    Thus, we need to ensure that $t_2 > t_1$ (and that $t_3 - t_2 > 1$; we assume this, as we can look at smaller level sets if necessary). To do this, we use that $t_2$ is an decreasing function of $\alpha$. Thus, $t_2 > t_1$ as soon as
    \begin{equation*}
        \alpha \leq \frac{1}{2}\sigma_1 \iota_1^2 \exp \p{- \frac{\log \csb{\frac{\norm{\theta_0 - \optEmp}_\HH^2}{\iota_1^2} 4n\kappa_{\TPop} + \frac{1}{\eta_b \sigma_1 - 1}}}{\log \frac{A_1}{\bar{A}_b}}},
    \end{equation*}
    which is exactly the purpose of the technical assumption $\alpha \leq \alpha_1$, with $\alpha_1$ defined in Def.~\ref{def:alpha_1}.
\end{proof}

\subsection{A lower bound for the estimator with small learning rate}
We now derive a similar result for $\theta_s$. To lighten the notations, we use $\mu_i^{(\eta_s, t_s)} = \mu_i$.

\begin{lemma}[Estimator with small step size.]\label{lem:technical_bound_small_lr}
    Set $\alpha > 0$ s.t.
    \begin{equation}
        \alpha < \alpha_1,
    \end{equation}
    where $\alpha_1$ is defined in Def.~\ref{def:alpha_1}. Define the quantity $\epsilon_s$ with
    \begin{equation}
        \epsilon_s^2 = \frac{\sum_{i < n} \mu_i^2}{\mu_n^2}.
    \end{equation}
    Then, running gradient descent on $\LEmp$ with step size $\eta_s$ until the $(\alpha/2, \alpha)$ level sets are reached, \emph{i.e.}
    \begin{equation}
        \frac{1}{2} \alpha \leq \LEmp(\theta_b) \leq \alpha,
    \end{equation}
    ensures that
    \begin{equation}
        \epsilon_s^2 \leq 1/(16n\kappa_{\TPop}), ~~ \andtext ~~ \frac{2}{5} \alpha \leq \frac{1}{2} \sigma_n \mu_n^2 \leq \alpha.
    \end{equation}
    The resulting estimator is obtained with $t_s$ steps, with
    \begin{equation}\label{eq:lower_bound_t_small}
        t_s \geq \frac{1}{2} \frac{\log \frac{1}{2}\frac{\sigma_n \iota_n^2}{\alpha}}{\log \frac{1}{1 - \eta_s \sigma_n}} \geq \frac{1}{2} \frac{\log \csb {\frac{\norm{\theta_0 - \optEmp}_\HH^2}{\iota_n^2} \max \cb{16n\kappa_{\TPop}, 4\kappa_\TEmp}}}{\log \frac{1 - \eta_s\sigma_{n}}{1 - \eta_s\sigma_{n-1}}}.
    \end{equation}
\end{lemma}

\begin{proof}
    The proof is very close to the one of Lemma~\ref{lem:technical_bound_big_lr}. We only give the main results.

    \paragraph{Bound on $\epsilon_s$.}
    This quantity can be written
    \begin{equation*}
        \epsilon_s^2 = \frac{\sum_{i < n} \mu_i^2}{\mu_n^2} = \frac{\sum_{i < n} \iota_i^2 A_i^{2t}}{\iota_n^2 A_n^{2t}} \leq \frac{\norm{\theta_0 - \optEmp}_\HH^2}{\iota_n^2} \p{\frac{\bar{A}_s}{A_n}}^{2t},
    \end{equation*}
    with $1 - \eta_s \sigma_{n-1} = \bar{A}_s > A_n = 1 - \eta_s \sigma_n$ following Asmpt.\ref{asmpt:lr}. Thus,
    \begin{equation*}
        \forall \delta > 0, ~~ t \geq t_1 \eqdef \frac{1}{2} \frac{\log \frac{\norm{\theta_0 - \optEmp}_\HH^2}{\delta^2 \iota_n^2}}{\log \frac{1 - \eta_s\sigma_{n}}{1 - \eta_s\sigma_{n-1}}} \implies \epsilon_s^2 \leq \delta^2.
    \end{equation*}

    \paragraph{Bound on $\mu_n$}
    Again, following the same paragraph in Lemma~\ref{lem:technical_bound_big_lr}, we have
    \begin{equation*}
        \LEmp(\theta_s) = \frac{1}{2} \sum_{i=1}^n \sigma_i \mu_i^2 = \frac{1}{2} \sigma_n \mu_n^2 \p{1 + \frac{\sum_{i > 1} \sigma_i \mu_i^2}{\sigma_n \mu_n^2}}
    \end{equation*}
    as we assumed that $\iota_n \neq 0$ in Asmpt.~\ref{asmpt:init_iota}. We bound the loss in two ways. First, by definition
    \begin{equation}\label{eq:proof_small_lr_1}
        \frac{1}{2} \sigma_n \mu_n^2 \leq \LEmp(\theta_s) \leq \frac{1}{2} \sigma_n \mu_n^2 (1 + \kappa_\TEmp \epsilon_s^2)
    \end{equation}
    and second, we assumed the estimator to belong to the $(\alpha/2, \alpha)$ level set of $\LEmp$, \emph{i.e.}
    \begin{equation}\label{eq:proof_small_lr_2}
        \frac{\alpha}{2} \leq \LEmp(\theta_s) \leq \alpha.
    \end{equation}
    Combining \cref{eq:proof_small_lr_1} with \cref{eq:proof_small_lr_2}, we have that
    \begin{equation}
        \frac{\alpha}{2 (1 + \kappa_\TEmp \epsilon_s^2)} \leq \frac{1}{2} \sigma_n \mu_n^2 \leq \alpha.
    \end{equation}

    \paragraph{Feasability.}
    The discussion is the same, but with the values
    \begin{align*}
        t_1 & = \frac{1}{2} \frac{\log \csb {\frac{\norm{\theta_0 - \optEmp}_\HH^2}{\iota_n^2} \max \cb{16n\kappa_{\TPop}, 4\kappa_\TEmp}}}{\log \frac{1 - \eta_s\sigma_{n}}{1 - \eta_s\sigma_{n-1}}} \\
        t_2 & = \frac{1}{2} \frac{\log \frac{1}{2}\frac{\sigma_n \iota_n^2}{\alpha}}{\log \frac{1}{1 - \eta_s \sigma_n}}                                                                              \\
        t_3 & = \frac{1}{2} \frac{\log \frac{5}{4}\frac{\sigma_n \iota_n^2}{\alpha}}{\log \frac{1}{1 - \eta_s \sigma_n}}.
    \end{align*}
    Note that the addition of $\kappa_\TEmp$ in the definition of $t_1$ is simply to ensure that
    \begin{equation*}
        \forall t \geq t_1, ~~ \epsilon_s^2 \leq \frac{1}{4 \kappa_\TEmp}, ~~ \text{so that} ~~ \frac{2}{5}\alpha \leq \frac{\alpha}{2 (1 + \kappa_\TEmp \epsilon_s^2)} \leq \frac{1}{2} \sigma_n \mu_n^2.
    \end{equation*}
    To ensure the feasibility of both bounds at the same time, we need to ensure $t_2 > t_1$. A sufficient condition for this is having
    \begin{equation*}
        \alpha \leq
        \frac{1}{2} \sigma_n \iota_n^2 \exp \p{- \frac{\log \csb {\frac{\norm{\theta_0 - \optEmp}_\HH^2}{\iota_n^2} \max \cb{16n\kappa_{\TPop}, 4\kappa_\TEmp}+ \frac{1}{1 - \eta_s \sigma_n}}}{\log \frac{1 - \eta_s\sigma_{n}}{1 - \eta_s\sigma_{n-1}}}}
    \end{equation*}
    which is, again, covered with the assumption $\alpha \leq \alpha_1$ defined in Def.~\ref{def:alpha_1}.
\end{proof}

\subsection{Plugging the two together}
\begin{theorem}
    Assume that the optimization error satisfy
    \begin{equation}
        \alpha \leq \alpha_1,
    \end{equation}
    where $\alpha_1$ is defined in Definition~\ref{def:alpha_1}. Assume that Assumption~\ref{asmpt:notation_operator} on the operators $\TEmp$ and $\TPop$ hold, that the condition on the learning rates $\eta_b, \eta_s$ of Assumption~\ref{asmpt:lr} hold, as the condition on the initialization in Assumption~\ref{asmpt:init_iota}.

    Assume that gradient descent is performed until the $(\alpha/2, \alpha)$ level sets are reached. Then we have that
    \begin{equation}
        \LPop(\theta_b) \leq c_\alpha \LPop(\theta_s), ~~ \text{with} ~~ c_\alpha = \csb{\frac{1 + 2\frac{\sigma_1}{\varsigma_1} \frac{\LPop(\optEmp)}{\alpha}}{\p{1 - \sqrt{18 \frac{\sigma_n}{\varsigma_n} \frac{\LPop(\optEmp)}{\alpha}}}_+}}.
    \end{equation}
    Further assume that Assumption~\ref{asmpt:target} holds. Then $c_\alpha \leq 2$ and the bound becomes
    \begin{equation}
        \LPop(\theta_b) \leq 34 \frac{\kappa_{\TPop}}{\kappa_\TEmp} \LPop(\theta_s).
    \end{equation}
\end{theorem}

\begin{proof}
    ~\paragraph{Upper bound for big learning rate.}
    We first proceed to bounding $\LPop(\theta_b)$. In this paragraph, we use $\mu_i = \mu_i^{\eta_b, t_b}$. We have
    \begin{align*}
        \LPop(\theta_b)
         & = \frac{1}{2} \norm{\theta_b - \optPop}_\TPop^2
         &                                                                                                                                  & \text{(Definition)}                               \\
         & \leq \norm{\theta_b - \optEmp}_\TPop^2 + \norm{\optEmp - \optPop}_\TPop^2
         &                                                                                                                                  & \text{(Triangular inequality)}                    \\
         & = \norm{\theta_b - \optEmp}_\TPop^2 + 2 \LPop(\optEmp)                                                                                                                               \\
         & \leq \norm{\mu_1 e_1}_\TPop^2 \p{1 + \frac{\norm{\sum_{i > 1} \mu_i e_i}_\TPop}{\norm{\mu_1 e_1}_\TPop}}^2 + 2 \LPop(\optEmp)
         &                                                                                                                                  & \text{(Triangular inequality)}                    \\
         & \leq 2\norm{\mu_1 e_1}_\TPop^2 \p{1 + \frac{\norm{\sum_{i > 1} \mu_i e_i}_\TPop^2}{\norm{\mu_1 e_1}_\TPop^2}} + 2 \LPop(\optEmp)
         &                                                                                                                                  & \text{($(a+b)^2 \leq 2 (a^2 + b^2)$)}             \\
         & \leq 2\norm{\mu_1 e_1}_\TPop^2 \p{1 + \kappa_{\TPop} \frac{\absv{\sum_{i > 1} \mu_i}^2}{\absv{\mu_1}^2}} + 2 \LPop(\optEmp)
         &                                                                                                                                  & \text{(Def. of $\kappa_\TPop$, $\norm{e_i} = 1$)} \\
         & \leq 2\norm{\mu_1 e_1}_\TPop^2 \p{1 + \kappa_{\TPop} n \epsilon_s^2} + 2 \LPop(\optEmp).
         &                                                                                                                                  & \text{(Cauchy-Schwartz)}
    \end{align*}
    Now, we use the results of Lemma~\ref{lem:technical_bound_big_lr}. Firstly, we have $\kappa_{\TPop} n \epsilon_{b}^2 \leq 1/4$. Secondly, we can use $\norm{e_1}_\TPop^2 \leq \varsigma_1$. The previous inequality then turns to
    \begin{equation*}
        \LPop(\theta_b)
        \leq \frac{5\varsigma_1}{2} \mu_1^2 + 2 \LPop(\optEmp).
    \end{equation*}
    Finally, we use the fact that $\frac{1}{2}\sigma_1 \mu_1^2 \leq \alpha$ to conclude with
    \begin{equation}\label{eq:thm_appendix_final_big}
        \LPop(\theta_b)
        \leq 5 \alpha \frac{\varsigma_1}{\sigma_1} + 2 \LPop(\optEmp).
    \end{equation}

    \paragraph{Lower bound for small learning rate.}
    We now turn to bounding $\LPop(\theta_s)$. Here, we use $\mu_i = \mu_i^{\eta_s, t_s}$. We have
    \begin{align*}
        \LPop(\theta_s)
         & = \frac{1}{2} \norm{\theta_s - \optPop}_\TPop^2
         &                                                                                                                                                                               & \text{(Definition)}                                    \\
         & \geq \frac{1}{2} \p{\norm{\theta_s - \optEmp}_\TPop - \norm{\optEmp - \optPop}_\TPop}^2
         &                                                                                                                                                                               & \text{(Triangular inequality)}                         \\
         & \geq \frac{1}{2} \p{\norm{\mu_n e_n}_\TPop - \norm{\sum_{i<n} \mu_i e_i}_\TPop - \norm{\optEmp - \optPop}_\TPop}^2
         &                                                                                                                                                                               & \text{(Idem)}                                          \\
         & = \frac{1}{2} \norm{\mu_n e_n}_\TPop^2 \p{1 - \frac{\norm{\sum_{i<n} \mu_i e_i}_\TPop}{\norm{\mu_n e_n}_\TPop} - \sqrt{\frac{2 \LPop(\optEmp)}{\norm{\mu_n e_n}_\TPop^2}}}^2                                                           \\
         & \geq \frac{1}{2} \norm{\mu_n e_n}_\TPop^2 \p{1 - 2\frac{\norm{\sum_{i<n} \mu_i e_i}_\TPop}{\norm{\mu_n e_n}_\TPop} - \sqrt{\frac{8\LPop(\optEmp)}{\norm{\mu_n e_n}_\TPop^2}}}
         &                                                                                                                                                                               & \text{(As $(1-x)^2 \geq 1 - 2x$)}                      \\
         & \geq \frac{1}{2} \varsigma_n \mu_n^2 \p{1 - 2 \sqrt{\kappa_{\TPop}}\frac{\absv{\sum_{i<n} \mu_i}}{\absv{\mu_n}} - \sqrt{\frac{8\LPop(\optEmp)}{\varsigma_n \mu_n^2}}}
         &                                                                                                                                                                               & \text{(Def. of $\kappa_\TPop$, with $\norm{e_i} = 1$)} \\
         & \geq \frac{1}{2} \varsigma_n \mu_n^2 \p{1 - 2 \sqrt{\kappa_{\TPop}} \p{\frac{n \sum_{i<n} \mu_i^2}{\mu_n^2}}^{1/2} - \sqrt{\frac{8\LPop(\optEmp)}{\varsigma_n \mu_n^2}}}
         &                                                                                                                                                                               & \text{(Cauchy-Schwartz)}                               \\
         & \geq \frac{1}{2} \varsigma_n \mu_n^2 \p{1 - 2 \sqrt{\kappa_{\TPop} n} \epsilon_s - \sqrt{\frac{8\LPop(\optEmp)}{\varsigma_n \mu_n^2}}}
         &                                                                                                                                                                               & \text{(Def. of $\epsilon_s$)}                          \\
    \end{align*}
    We then use Lemma~\ref{lem:technical_bound_small_lr}. Firstly, we can use $\epsilon_s^2 \leq 1/(16n\kappa_{\TPop})$ so that $2 \sqrt{\kappa_{\TPop} n} \epsilon_s \leq 1/4$. This gives
    \begin{equation*}
        \LPop(\theta_s)
        \geq \frac{3}{8} \varsigma_n \mu_n^2 \p{1 - \frac{4}{3}\sqrt{\frac{8\LPop(\optEmp)}{\varsigma_n \mu_n^2}}}.
    \end{equation*}
    Secondly, we have $\sigma_n \mu_n^2/2 \geq 2\alpha/5$. This give ultimately
    \begin{equation*}
        \LPop(\theta_s)
        \geq \frac{3}{10} \alpha \frac{\varsigma_n}{\sigma_n} \p{1 - \frac{4}{3}\sqrt{\frac{8\LPop(\optEmp)}{\frac{4}{5} \alpha \frac{\varsigma_n}{\sigma_n}}}}
        = \frac{3}{10} \alpha \frac{\varsigma_n}{\sigma_n} \p{1 - \sqrt{\frac{160}{9} \frac{\LPop(\optEmp)}{\alpha \frac{\varsigma_n}{\sigma_n}}}}.
    \end{equation*}
    Simplifying this expression gives
    \begin{equation}\label{eq:thm_appendix_final_small}
        \LPop(\theta_s)
        \geq \frac{3}{10} \alpha \frac{\varsigma_n}{\sigma_n} \p{1 - \sqrt{18 \frac{\LPop(\optEmp)}{\alpha \frac{\varsigma_n}{\sigma_n}}}}.
    \end{equation}

    \paragraph{Combining the two bounds.}
    We now simply combine the upper bound of \cref{eq:thm_appendix_final_big} and the lower bound of \cref{eq:thm_appendix_final_big}. We get
    \begin{equation}
        \frac{\LPop(\theta_s)}{\LPop(\theta_b)} \geq
        \frac{3}{50}
        \frac{\kappa_\TEmp}{\kappa_{\TPop}}
        \csb{
            \frac{1 - \sqrt{18 \frac{\sigma_n}{\varsigma_n} \frac{\LPop(\optEmp)}{\alpha}}}
            {1 + 2\frac{\sigma_1}{\varsigma_1} \frac{\LPop(\optEmp)}{\alpha}}
        }.
    \end{equation}
    We may prefer the other form, introducing the positive part $(x)_+ = \max(0, x)$ and using $50/3 < 17$:
    \begin{equation}
        \LPop(\theta_b) \leq 17 \frac{\kappa_{\TPop}}{\kappa_\TEmp} \csb{\frac{1 + 2\frac{\sigma_1}{\varsigma_1} \frac{\LPop(\optEmp)}{\alpha}}{\p{1 - \sqrt{18 \frac{\sigma_n}{\varsigma_n} \frac{\LPop(\optEmp)}{\alpha}}}_+}} \LPop(\theta_s)
        \eqdef 17 \frac{\kappa_{\TPop}}{\kappa_\TEmp} c_\alpha \LPop(\theta_s).
    \end{equation}
    Finally, with Assumption~\ref{asmpt:target} we have
    \begin{align*}
        1 + 2\frac{\sigma_1}{\varsigma_1} \frac{\LPop(\optEmp)}{\alpha} \leq \frac{3}{2}, \\
        1 - \sqrt{18 \frac{\sigma_n}{\varsigma_n} \frac{\LPop(\optEmp)}{\alpha}} \geq \frac{1}{2},
    \end{align*}
    so that $c_\alpha \leq 2$.
\end{proof}

%% file: sections/useful_operators.tex
\subsection{Useful operators}
\label{sec:appendix_operators}
We assume there are $n$ training samples. If considered, the test loss consists of $m$ samples. 

We denote $\SEmp, \SEmpD$ the \emph{sampling} operator and its dual, which are defined as
\begin{equation}
    \begin{aligned}
        \SEmp: \HH \to \RR^n, && \forall f \in \HH, ~~ &\SEmp(f) = \frac{1}{\sqrt{n}}\p{\begin{matrix}
            \dotprod{f}{\phi(x_1)}_{\HH} \\ \vdots \\ \dotprod{f}{\phi(x_n)}_{\HH}
        \end{matrix}} \\
        \SEmpD: \RR^n \to \HH, && \forall \alpha \in \RR^n, ~~ &\SEmpD(\alpha) = \frac{1}{\sqrt{n}} \sum_{i=1}^n \alpha_i \phi(x_i),
    \end{aligned}
\end{equation}
so that the covariance operator $\TEmp = \SEmpD \SEmp$ and the kernel matrix $K$ write 
\begin{equation}
    \begin{aligned}
        \TEmp: \HH \to \HH, && \TEmp &= \SEmpD \SEmp \\
        K/n: \RR^n \to \RR^n, && \frac{K}{n} &= \SEmp \SEmpD.
    \end{aligned}
\end{equation}
The population version are $\SPop, \SPopD, \TPop$. There are written with an expectation, or with the test dataset as a proxy. 

Note that we have $\TEmp^{-1} \SEmpD = \SEmpD (K/n)^{-1}$. We denote $\sigma_i, e_i$ the spectrum of $\TEmp$ and $u_i$ the eigenvectors of $K/n$ (not $K$), with the same spectrum. The eigenvectors in $\HH$ and $\RR^n$ are related with
\begin{equation*}
    \forall i \in \cb{1, \dots, n}, ~~ u_i = \frac{1}{\sqrt{\sigma_i}} \SEmp e_i, ~~ e_i = \frac{1}{\sqrt{\sigma_i}} \SEmpD u_i.
\end{equation*}

Finally, an estimator $\theta \in \HH$ can be represented with a vector $\alpha \in \RR^n$. Specifically, we have the relation 
\begin{equation}\label{eq:relation_alpha_theta}
    \theta = \sqrt{n} \SEmpD \alpha \iff \sqrt{n} \SEmp \theta = K \alpha \iff \alpha = \sqrt{n} K^{-1} \SEmp \theta.
\end{equation}